%% file: main.tex
\documentclass{article}

\usepackage[preprint]{neurips_2023}

\usepackage[utf8]{inputenc} 
\usepackage[T1]{fontenc}    
\usepackage{hyperref}       
\usepackage{url}            
\usepackage{booktabs}       
\usepackage{amsfonts}       
\usepackage{nicefrac}       
\usepackage{microtype}      
\usepackage[dvipsnames]{xcolor}        

\usepackage{multirow}

\usepackage{enumitem}
\usepackage{bm}%
\usepackage{amsmath}%

\usepackage{amssymb}%
\usepackage{amsthm}
\usepackage{colonequals}
\usepackage{psfrag}
\usepackage{mathrsfs}

\usepackage{bbding}
\usepackage{mathtools}

\usepackage{booktabs,siunitx}

\usepackage{algorithm}
\usepackage{algpseudocode}

\usepackage[capitalize,noabbrev]{cleveref}

\hypersetup{%
  colorlinks=true,
  urlcolor=blue,%
  linkcolor=blue,%
  citecolor=blue,%
  linkbordercolor=blue,
  pdfborderstyle={/S/U/W 1}
}

\input{notation.tex}

\title{Aleatoric and Epistemic Discrimination:\\
Fundamental Limits of Fairness Interventions}

\author{%
    Hao Wang\\
    MIT-IBM Watson AI Lab\\
    \texttt{hao@ibm.com} \\
    \And
    Luxi He\\
    Harvard College\\
    \texttt{luxihe@college.harvard.edu}\\
    \AND
    Rui Gao\\
    The University of Texas at Austin\\
    \texttt{rui.gao@mccombs.utexas.edu} \\
    \And
    Flavio P. Calmon\\
    Harvard University\\
    \texttt{flavio@seas.harvard.edu} 
}

\begin{document}

\maketitle

\begin{abstract}
Machine learning (ML) models can underperform on certain population groups due to choices made during model development and bias inherent in the data. We categorize sources of discrimination in the ML pipeline into two classes: \emph{aleatoric discrimination}, which is inherent in the data distribution, and \emph{epistemic discrimination}, which is due to decisions made during model development. We quantify aleatoric discrimination by determining the performance limits of a model under fairness constraints, assuming perfect knowledge of the data distribution. We demonstrate how to characterize aleatoric discrimination by applying Blackwell's results on comparing statistical experiments. We then quantify epistemic discrimination as the gap between a model's accuracy when fairness constraints are applied and the limit posed by aleatoric discrimination. We apply this approach to benchmark existing fairness interventions and investigate fairness risks in data with missing values. Our results indicate that state-of-the-art fairness interventions are effective at removing epistemic discrimination on standard (overused) tabular datasets. However, when data has missing values, there is still significant room for improvement in handling aleatoric discrimination.
\end{abstract}

\input{Section}

\clearpage

\section*{Acknowledgement}

\vspace{-2mm}

This material is based upon work supported by the National Science Foundation under grants CAREER 1845852, CIF 2312667, FAI 2040880, CIF 1900750. 

\bibliography{references}
\bibliographystyle{apalike}

\newpage
\appendix
\onecolumn

\input{Appendix.tex}

\end{document}

%% file: notation.tex
\newtheorem{thm}{Theorem}
\newtheorem{lem}{Lemma}

\newtheorem{prop}{Proposition}

\theoremstyle{definition}

\newtheorem{defn}{Definition}

\newtheorem{rem}{Remark}

\DeclareMathOperator*{\argmin}{\arg\!\min}

\newcommand{\ba}{\bm{a}}

\newcommand{\bp}{\bm{p}}

\newcommand{\bv}{\bm{v}}

\newcommand{\bg}{\bm{g}}

\newcommand{\bM}{\bm{M}}

\newcommand{\bP}{\bm{P}}
\newcommand{\bQ}{\bm{Q}}

\newcommand{\blambda}{\pmb{\lambda}}
\newcommand{\bmu}{\pmb{\mu}}

\newcommand{\bLambda}{\pmb{\Lambda}}

\newcommand{\rB}{\textnormal{B}}

\newcommand{\rX}{\textnormal{X}}
\newcommand{\rY}{\textnormal{Y}}
\newcommand{\rZ}{\textnormal{Z}}
\newcommand{\rA}{\textnormal{A}}
\newcommand{\rS}{\textnormal{S}}

\newcommand{\diag}{\mathsf{diag}}

\newcommand{\ones}{\bm{1}}

\newcommand{\Reals}{\mathbb{R}}
\newcommand{\defined}{\triangleq}

\newcommand{\sto}{\mbox{\normalfont s.t.}}

\newcommand{\indicator}[1]{\mathbb{I}_{#1}}

\newcommand{\dif}{\textrm{d}}

\newcommand{\ExpVal}[2]{\mathbb{E}\left[ #2 \right]}
\newcommand{\EE}[1]{\ExpVal{}{#1}}
\newcommand{\EEE}[2]{\mathbb{E}_{#1}\left[ #2 \right]}

\newcommand{\indep}{\rotatebox[origin=c]{90}{$\models$}}

\renewcommand{\CheckmarkBold}{{\color{ForestGreen}\Checkmark}}
\renewcommand{\XSolidBold}{{\color{BrickRed}\XSolid}}

%% file: Section.tex
\section{Introduction}

Algorithmic discrimination may occur in different stages of the machine learning (ML) pipeline. For example, historical biases in the data-generating process can propagate to  downstream tasks; human biases can influence a ML model through inductive bias; optimizing solely for accuracy can lead to disparate model performance across groups in the data \citep{suresh2019framework,mayson2019bias}.  The past years have seen a rapid increase in algorithmic interventions that aim to mitigate biases in ML models \citep[see e.g.,][]{zemel2013learning,feldman2015certifying,calmon2017optimized,menon2018cost,zhang2018mitigating,zafar2019fairness,friedler2019comparative,bellamy2019ai,kim2019multiaccuracy,celis2019classification,yang2020fairness,jiang2020identifying,jiang2020wasserstein,martinez2020minimax,lowy2021fermi,alghamdi2022beyond}. A recent survey \citep{hort2022bia} found \emph{nearly 400} fairness-intervention algorithms, including 123 pre-processing, 212 in-processing, and 56 post-processing algorithms introduced in the past decade.

Which sources of biases are (the hundreds of) existing fairness interventions trying to control? In order to create effective strategies for reducing algorithmic discrimination, it is critical to disentangle where biases in model performance originate. 
For instance, if a certain population group has significantly more missing features in training data, then it is more beneficial to collect data than selecting a more complex model class or training strategy. Conversely, if the model class does not accurately represent the underlying distribution of a specific population group, then collecting more data for that group will not resolve performance disparities.

We divide algorithmic discrimination\footnote{There are various measures to quantify algorithmic discrimination, and the choice should be based on the specific application of interest \citep[see][for a more detailed discussion]{blodgett2020language,varshney2021trustworthy,katzman2023taxonomizing}. In this paper, we focus on group fairness measures (see Table~\ref{tabel:FairMetric} for some examples), which are crucial in contexts like hiring and recidivism prediction.} into two categories: aleatoric and epistemic discrimination.\footnote{We borrow this notion from ML uncertainty literature \citep[see][for a survey]{hullermeier2021aleatoric} and defer a detailed comparison in Appendix~\ref{subsec::related_uncertainty}.}
Aleatoric discrimination captures inherent biases in the data distribution that can lead to  unfair decisions in  downstream tasks. Epistemic discrimination, in turn, is due to algorithmic choices made during model development and lack of knowledge about the optimal ``fair'' predictive model.

In this paper, we provide methods for measuring aleatoric and epistemic discrimination in classification tasks for group fairness metrics. Since aleatoric discrimination only depends on properties of the data distribution and the fairness measure of choice, we quantify it by asking a fundamental question: 

\begin{center}
\emph{For a given data distribution, what is the best achievable performance (e.g., accuracy)\\under a set of group fairness constraints?} 
\end{center}

We refer to the answer as the \emph{fairness Pareto frontier}. This frontier delineates the optimal performance achievable by a classifier when  unlimited data and computing power are available. For a fixed data distribution, the fairness Pareto frontier represents the ultimate, information-theoretic limit for accuracy and group fairness beyond which no model can achieve. Characterizing this limit enables us to  (i) separate sources of discrimination and create strategies to control them accordingly; (ii) evaluate the effectiveness of existing fairness interventions for reducing epistemic discrimination; and (iii) inform the development of data collection methods that promote fairness in downstream tasks.

At first, computing the fairness Pareto frontier can appear to be an intractable problem since it requires searching over all possible classifiers---even if the data distribution is known exactly. Our main technical contribution is to provide an upper bound estimate for this frontier by solving a sequence of optimization problems. The proof technique is based on Blackwell's seminal results \citep{blackwell1953equivalent}, which proposed the notion of comparisons of statistical experiments and inspired a line of works introducing alternative comparison criteria \citep[see e.g.,][]{shannon1958note,le1964sufficiency,torgersen1991comparison,cohen1998comparisons,raginsky2011shannon}. Here, we apply these results to develop an algorithm that iteratively refines the achievable fairness Pareto frontier. We also prove convergence guarantees for our algorithm and demonstrate how it can be used to benchmark existing fairness interventions.

We quantify epistemic discrimination by comparing a classifier's performance with the information-theoretic optimal given by the fairness Pareto frontier. Our experiments indicate that given sufficient data, state-of-the-art (SOTA) group fairness interventions are effective at reducing epistemic discrimination as their gap to the information-theoretic limit is small (see Figure~\ref{Fig::BayesOptimal} and \ref{Fig::Fair_Frontier}). 
Consequently, there are diminishing returns in benchmarking new fairness interventions on standard (overused) tabular datasets (e.g., UCI Adult and ProPublica COMPAS datasets).
However, existing interventions \emph{do not} eliminate aleatoric discrimination as this type of discrimination is not caused by choice of learning algorithm or model class, and is due to the data distribution. 
Factors such as data missing values can significantly contribute to aleatoric discrimination. We observe that when population groups have disparate missing patterns, aleatoric discrimination escalates, leading to a sharp decline in the effectiveness of fairness intervention algorithms (see Figure~\ref{Fig::Reduce_Aleatoric}).

\subsection*{Related Work}

There is significant work analyzing the tension between group fairness measures and model performance metrics \citep[see e.g.,][]{kleinberg2016inherent,chouldechova2017fair,corbett2017algorithmic,chen2018my,wick2019unlocking,dutta2020there,wang2021split}. 
For example, there is a growing body of work on omnipredictors \citep{gopalan2021omnipredictors,hu2023omnipredictors,globus2023multicalibrated} discussing how, and under which conditions, the fair Bayes optimal classifier can be derived using post-processing techniques from multicalibrated regressors.
While previous studies \citep{hardt2016equality,corbett2017algorithmic,menon2018cost,chzhen2019leveraging,yang2020fairness,zeng2022bayes,zeng2022fair} have investigated the fairness Pareto frontier and fair Bayes optimal classifier, our approach differs from this prior work in the following aspects: our approach is applicable to \emph{multiclass} classification problems with \emph{multiple} protected groups; it \emph{avoids disparate treatment} by not requiring the classifier to use group attributes as an input variable; and it can handle \emph{multiple} fairness constraints simultaneously and produce fairness-accuracy trade-off curves (instead of a single point). Additionally, our proof techniques based on Blackwell's results on comparing statistical experiments are unique and may be of particular interest to fair ML and information theory communities. We present a detailed comparison with this line of work in Table~\ref{tabel:comparison} of Appendix~\ref{append:future}.

We recast the fairness Pareto frontier in terms of the conditional distribution $P_{\hat{\rY}|\rY,\rS}$ of predicted outcome $\hat{\rY}$ given true label $\rY$ and group attributes $\rS$. This conditional distribution is related to confusion matrices 
conditioned on each subgroup. In this regard, our work is related to \citet{verma2018fairness,alghamdi2020model,kim2020fact,yang2020fairness,berk2021fairness}, which observed that many group fairness metrics can be written in terms of the confusion matrices for each subgroup. Among them, the closest work to ours is \citet{kim2020fact}, which optimized accuracy and fairness objectives over these confusion matrices and proposed a post-processing technique for training fair classifiers. However, they only imposed marginal sum constraints for the confusion matrices. We demonstrate that the feasible region of confusion matrices can be much smaller (see Remark~\ref{rem::ex_ind} for an example), leading to a tighter approximation of the fairness Pareto frontier.

Recently, many strategies have been proposed to reduce the tension between group fairness and model performance by investigating properties of the data distribution. For example, \citet{blum2019recovering,suresh2019framework,fogliato2020fairness,wang2020robust,mehrotra2021mitigating,fernando2021missing,wang2021analyzing,zhang2021assessing,tomasev2021fairness,jacobs2021measurement,kallus2022assessing,jeong2022fairness} studied how noisy or missing data affect fairness and model accuracy. \citet{dwork2018decoupled,ustun2019fairness,wang2021split} considered training a separate classifier for each subgroup when their data distributions are different. Another line of research introduces data pre-processing techniques that manipulate data distribution for reducing its bias \citep[e.g.,][]{calmon2017optimized,kamiran2012data}. Among all these works, the closest one to ours is \citet{chen2018my}, which decomposed group fairness measures into bias, variance, and noise (see their Theorem~1) and proposed strategies for reducing each term. Compared with \citet{chen2018my}, the main difference is that we characterize a fairness Pareto frontier that depends on fairness metrics \emph{and} a performance measure, giving a complete picture of how the data distribution influences fairness and accuracy.

\section{Preliminaries}
Next, we introduce notation, overview the key results in \citet{blackwell1953equivalent} on comparisons of experiments, and outline the fair classification setup considered in this paper.

\paragraph{Notation.} For a positive integer $n$, let $[n]\defined \{1,\cdots,n\}$. We denote all probability distributions on the set $\mathcal{X}$ by $\mathcal{P}(\mathcal{X})$. Moreover, we define the probability simplex $\Delta_m \defined \mathcal{P}([m])$. When random variables $\rA$, $\rX$, $\rZ$ form a Markov chain, we write $\rA -\rX - \rZ$. We write the mutual information between $\rA$, $\rX$ as $I(\rA;\rX) \defined \EEE{P_{\rA,\rX}}{\log\frac{P_{\rA,\rX}(\rA,\rX)}{P_{\rA}(\rA)P_{\rX}(\rX)}}$. Since $I(\rA;\rX)$ is determined by the marginal distribution $P_{\rA}$ and the conditional distribution $P_{\rX|\rA}$, we also write $I(\rA;\rX)$ as $I(P_{\rA};P_{\rX|\rA})$. When $\rA$, $\rX$ are independent, we write $\rA \indep \rX$.

If a random variable $\rA \in [n]$ has finite support, the conditional distribution $P_{\rX|\rA}: [n] \to \mathcal{P}(\mathcal{X})$ can be equivalently written as $\bP \defined (P_1,\cdots,P_n)$ where each $P_i = P_{\rX|\rA=i} \in \mathcal{P}(\mathcal{X})$. Additionally, if $\mathcal{X}$ is a finite set $[m]$, then $P_{\rX|\rA}$ can be fully characterized by a transition matrix. We use $\mathcal{T}(m|n)$ to denote all transition matrices from $[n]$ to $[m]$: $\left\{\bP\in \Reals^{n\times m} \ \Big| \ 0 \leq P_{i,j}\leq 1, \sum_{j=1}^m P_{i,j} = 1, \forall i\in [n]\right\}$.

\subsection*{Comparisons of Experiments}
\label{sec::comp_exp}

Given two statistical experiments (i.e., conditional distributions) $\bP$ and $\bQ$, is there a way to decide which one is more informative? Here $\bP$ and $\bQ$ have the common input alphabet $[n]$ and potentially different output spaces. Blackwell gave an answer in his seminal work \citep{blackwell1953equivalent} from a decision-theoretic perspective. We review these results next. 

Let $\mathcal{A}$ be a closed, bounded, convex subset of $\Reals^n$. A decision function $\bm{f}(x) = (a_1(x),\cdots, a_n(x))$ is any mapping from $\mathcal{X}$ to $\mathcal{A}$. It is associated with a loss vector:
\begin{align}
    \bv(\bm{f}) 
    = \left(\int a_1(x) \dif P_1(x),\cdots, \int a_n(x) \dif P_n(x) \right).
\end{align}
The collection of all $\bv(\bm{f})$ is denoted by $\mathcal{B}(\bP,\mathcal{A})$. Blackwell defined that $\bP$ is more informative than $\bQ$ if for every $\mathcal{A}$, $\mathcal{B}(\bP,\mathcal{A})\supseteq \mathcal{B}(\bQ,\mathcal{A})$. Intuitively, this result means any risk achievable with $\bQ$ is also achievable with $\bP$. Moreover, Blackwell considered the standard measure $P^*$ which is the probability distribution of $\bp(\bar{\rX})$ where $\bp(x): \mathcal{X} \to \Delta_n$ is a function defined as
\begin{align}
    \left(\frac{\dif P_1}{\dif P_1 + \cdots + \dif P_n} ,\cdots, \frac{\dif P_n}{\dif P_1 + \cdots + \dif P_n}\right).
\end{align} 
and $\bar{\rX}$ follows the probability distribution $\frac{P_1 + \cdots + P_n}{n}$. One of the most important findings by Blackwell in his paper is to discover the following equivalent conditions.
\begin{lem}[\citet{blackwell1951comparison,blackwell1953equivalent}]
\label{lem::blackwell_eqv}
The following three conditions are equivalent:
\begin{itemize}
    \item $\bP$ is more informative than $\bQ$;
    \item for any continuous and convex function $\phi:\Delta_n \to \Reals$, $\int \phi(\bp) \dif P^*(\bp) \geq \int \phi(\bp) \dif Q^*(\bp)$;
    \item there is a stochastic transformation $\mathsf{T}$ s.t. $\mathsf{T}P_i = Q_i$. In other words, there exists a Markov chain $\rA - \rX - \rZ$ for any distributions on $\rA$ such that $\bP = P_{\rX|\rA}$ and $\bQ = P_{\rZ|\rA}$. 
\end{itemize}
\end{lem}

If $\bP = P_{\rX|\rA}$ is more informative than $\bQ = P_{\rZ|\rA}$, by the third condition of Lemma~\ref{lem::blackwell_eqv} and the data processing inequality, $I(P_{\rA};P_{\rX|\rA})
\geq I(P_{\rA};P_{\rZ|\rA})$ holds for \emph{any} marginal distribution $P_{\rA}$. However, the converse does not hold in general---even if the above inequality holds for any $P_{\rA}$, $\bP$ is not necessarily more informative than $\bQ$  \citep{rauh2017coarse}. In this regard, Blackwell's conditions are ``stronger'' than the mutual information based data processing inequality.

\subsection*{Group Fair Classification}

\newcommand{\sccell}[2]{\setlength{\tabcolsep}{0pt}{\begin{tabular}{#1}#2 \end{tabular}}}

\begin{table*}[t]
\small
\centering
\resizebox{0.94\textwidth}{!}{
\renewcommand{\arraystretch}{1.25}
\begin{tabular}{lll}

\toprule

\textsc{Fairness Metric}

&

\textsc{Abbr.}

& 
\sccell{l}{
\textsc{Definition}\\
\textsc{Expression w.r.t.} $\bP$ 
}
\\
\toprule

Statistical Parity 

&

$\mathsf{SP}\leq \alpha_{\scalebox{.5}{\textnormal SP}}$

&

\sccell{l}{
$|\Pr(\hat{\rY}=\hat{y}|\rS=s)-\Pr(\hat{\rY}=\hat{y}|\rS=s')|\leq \alpha_{\scalebox{.5}{\textnormal SP}}$\\
$\left|\sum_{y=1}^{C} \left(\frac{\mu_{s,y}}{\mu_{s}} P_{(s,y),\hat{y}} - \frac{\mu_{s',y}}{\mu_{s'}} P_{(s',y),\hat{y}}\right) \right| \leq \alpha_{\scalebox{.5}{\textnormal SP}}$
}

\\ \midrule

Equalized Odds

&

$\mathsf{EO}\leq \alpha_{\scalebox{.5}{\textnormal EO}}$

& 

\sccell{l}{
$|\Pr(\hat{\rY} = \hat{y} | \rS=s,\rY=y) - \Pr(\hat{\rY} = \hat{y} | \rS=s',\rY=y)| \leq \alpha_{\scalebox{.5}{\textnormal EO}}$ \\
$\left|P_{(s,y),\hat{y}} - P_{(s',y),\hat{y}} \right|\leq \alpha_{\scalebox{.5}{\textnormal EO}}$
}
\\ \midrule

Overall Accuracy Equality

&
$\mathsf{OAE}\leq \alpha_{\scalebox{.5}{\textnormal OAE}}$

& 

\sccell{l}{
$|\Pr(\hat{\rY} = \rY|\rS=s)-\Pr(\hat{\rY}=\rY|\rS=s')|\leq \alpha_{\scalebox{.5}{\textnormal OAE}}$ \\
$\left|\sum_{y=1}^C \left( \frac{\mu_{s,y}}{\mu_s} P_{(s,y),y} - \frac{\mu_{s',y}}{\mu_{s'}} P_{(s',y),y} \right)\right|\leq \alpha_{\scalebox{.5}{\textnormal OAE}}$
}

\\

\bottomrule 
\end{tabular}
}
\caption{Standard group fairness metrics under multi-group and multi-class classification tasks. Here $\alpha_{\scalebox{.5}{\textnormal SP}}, \alpha_{\scalebox{.5}{\textnormal EO}}, \alpha_{\scalebox{.5}{\textnormal OAE}},\in [0,1]$ are threshold parameters, $\hat{y},y \in [C]$, $s,s'\in[A]$, and $\mu_{s,y}$, $\mu_{s}$ are defined in Proposition~\ref{prop::FATO_linear}. Our analysis can be extended to many other group fairness metrics \citep[see e.g., Table~1 in][]{kim2020fact}.}
\label{tabel:FairMetric}
\end{table*}

Consider a multi-class classification task, where the goal is to train a probabilistic classifier $h:\mathcal{X}\to \Delta_C$ that uses input features $\rX$ to predict their true label $\rY \in [C]$. Additionally, assume the classifier produces a predicted outcome $\hat{\rY} \in [C]$ and let $\rS\in [A]$ represent group attributes (e.g., race and sex). Depending on the domain of interest, $\rX$ can either include or exclude $\rS$ as an input to the classifier. Our framework can be easily extended to the setting where multiple subgroups overlap \citep{kearns2018preventing}. Throughout this paper, we focus on three standard group fairness measures: statistical parity ($\mathsf{SP}$) \citep{feldman2015certifying}, equalized odds ($\mathsf{EO}$) \citep{hardt2016equality,pleiss2017fairness}, and overall accuracy equality ($\mathsf{OAE}$) \citep{berk2021fairness} (see Table~\ref{tabel:FairMetric} for their definitions) but our analysis can be extended to many other group fairness metrics, including the ones in Table~1 of \citet{kim2020fact}, as well as alternative performance measures beyond accuracy.

\section{Fairness Pareto Frontier}

In this section, we introduce our main concept---fairness Pareto frontier ($\mathsf{FairFront}$). We use it to measure aleatoric discrimination and quantify epistemic discrimination by comparing a classifier's performance to the $\mathsf{FairFront}$. We recast $\mathsf{FairFront}$ in terms of the conditional distribution $P_{\hat{\rY}|\rS,\rY}$ and apply Blackwell's conditions to characterize the feasible region of this conditional distribution. This effort converts a functional optimization problem into a convex program with a small number of variables. However, this convex program may involve infinitely many constraints. Hence, we introduce a greedy improvement algorithm that iteratively refines the approximation of $\mathsf{FairFront}$ and tightens the feasible region of $P_{\hat{\rY}|\rS,\rY}$. Finally, we establish a convergence guarantee for our algorithm.

Recall that we refer to aleatoric discrimination as the inherent biases of the data distribution that can lead to an unfair or inaccurate classifier. As its definition suggests, aleatoric discrimination only relies on properties of the data distribution and fairness metric of choice---it does not depend on the hypothesis class nor optimization method. Below we introduce $\mathsf{FairFront}$ that delineates a curve of optimal accuracy over all probabilistic classifiers under certain fairness constraints for a given data distribution $P_{\rS,\rX,\rY}$. 
We use $\mathsf{FairFront}$ to quantify aleatoric discrimination. 
\begin{defn}
\label{defn::FATO}
For $\alpha_{\scalebox{.5}{\textnormal SP}}, \alpha_{\scalebox{.5}{\textnormal EO}}, \alpha_{\scalebox{.5}{\textnormal OAE}} \geq 0$ and a given $P_{\rS,\rX,\rY}$, we define $\mathsf{FairFront}(\alpha_{\scalebox{.5}{\textnormal SP}}, \alpha_{\scalebox{.5}{\textnormal EO}}, \alpha_{\scalebox{.5}{\textnormal OAE}})$ by
\begin{subequations}
\label{eq::FAO}
\begin{align}
    \mathsf{FairFront}(\alpha_{\scalebox{.5}{\textnormal SP}}, \alpha_{\scalebox{.5}{\textnormal EO}}, \alpha_{\scalebox{.5}{\textnormal OAE}}) 
    \defined \max_{h}~&\EE{\indicator{\hat{\rY}=\rY}}\\
    \sto~& \mathsf{SP} \leq \alpha_{\scalebox{.5}{\textnormal SP}}, \mathsf{EO} \leq \alpha_{\scalebox{.5}{\textnormal EO}}, \mathsf{OAE} \leq \alpha_{\scalebox{.5}{\textnormal OAE}}
\end{align}
\end{subequations}
where $\mathbb{I}$ is the indicator function; $\hat{\rY}$ is produced by applying the classifier $h$ to $\rX$; the maximum is taken over all measurable $h$; and the definitions of $\mathsf{SP}$, $\mathsf{EO}$, and $\mathsf{OAE}$ are in Table~\ref{tabel:FairMetric}. As a special case, if $\alpha_{\scalebox{.5}{\textnormal SP}}, \alpha_{\scalebox{.5}{\textnormal EO}}, \alpha_{\scalebox{.5}{\textnormal OAE}} \geq 1$, then $\mathsf{FairFront}(\alpha_{\scalebox{.5}{\textnormal SP}}, \alpha_{\scalebox{.5}{\textnormal EO}}, \alpha_{\scalebox{.5}{\textnormal OAE}})$ is the accuracy of the Bayes optimal classifier.
\end{defn}
Solving this functional optimization problem is difficult since it optimizes over all measurable classifiers. There is a line of works that proposed different fairness-intervention algorithms for training group-fair classifiers \citep[see e.g.,][]{menon2018cost,zhang2018mitigating,zafar2019fairness,celis2019classification,yang2020fairness,wei2021optimized,alghamdi2022beyond}. They restrict the model class and vary loss functions and optimizers to find classifiers that approach $\mathsf{FairFront}$ as close as possible. 
However, these algorithms only describe a lower bound for $\mathsf{FairFront}$. They do not determine what is the \emph{best} achievable accuracy for a given set of fairness constraints.

We circumvent the above-mentioned challenges by rewriting $\mathsf{FairFront}$ in terms of the conditional distribution $P_{\hat{\rY}|\rS,\rY}$. The caveat is that although each classifier yields a $P_{\hat{\rY}|\rS,\rY}$, not every conditional distribution corresponds to a valid classifier. Hence, we introduce the following definition which characterizes all feasible $P_{\hat{\rY}|\rS,\rY}$. 
\begin{defn}
\label{defn::achi_trans_mat}
Given $P_{\rX|\rS,\rY}$, we define $\mathcal{C}$ as the set of all conditional distributions $P_{\hat{\rY}|\rS,\rY}$ where $\hat{\rY}$ is produced by some probabilistic classifier $h$. In other words,
\begin{align}
    \mathcal{C}
    \defined \{P_{\hat{\rY}|\rS,\rY} \mid (\rS,\rY) - \rX -\hat{\rY}\}.
\end{align}
Throughout this paper, we write $P_{\hat{\rY}|\rS,\rY}$ or its corresponding transition matrix $\bP \in \mathcal{T}(C|AC)$ interchangeably. Specifically, the $(C(s-1)+y)$-th row, $\hat{y}$-th column of $\bP$ represents $P_{\hat{\rY}|\rS,\rY}(\hat{y}|s,y)$ and we denote it by $P_{(s,y),\hat{y}}$.
\end{defn}

\begin{rem}
\label{rem::trans_conf}
We demonstrate the connection between the conditional distribution $P_{\hat{\rY}|\rS,\rY}$ and confusion matrices in the setting of binary classification with binary groups. We define $\hat{\mathcal{C}}$ as the counterpart of $\mathcal{C}$ when we replace $P_{\rX|\rS,\rY}$ with an empirical distribution $\hat{P}_{\rX|\rS,\rY}$ computed from a dataset. The confusion matrix for group $s\in \{0,1\}$ consists of four numbers: True Positive ($\mathsf{TP}_{s}$), False Positive ($\mathsf{FP}_{s}$), False Negative ($\mathsf{FN}_{s}$), True Negative ($\mathsf{TN}_{s}$). Assume that the number of positive-label data $n_s^{+} = \mathsf{TP}_{s} + \mathsf{FN}_{s}$ and negative-label data $n_s^{-} = \mathsf{TN}_{s} + \mathsf{FP}_{s}$ are given---these numbers do not depend on the classifier. Then there is a one-to-one mapping from each element in $\hat{\mathcal{C}}$ to a confusion matrix:
\begin{align*}
    &\hat{P}_{\hat{\rY}|\rS,\rY}(1|s,1) 
    = \frac{1}{n_s^{+}}\mathsf{TP}_{s}, \quad 
    \hat{P}_{\hat{\rY}|\rS,\rY}(1|s,0) 
    = \frac{1}{n_s^{-}} \mathsf{FP}_{s},\\
    &\hat{P}_{\hat{\rY}|\rS,\rY}(0|s,1)
    = \frac{1}{n_s^{+}} \mathsf{FN}_{s}, \quad \hat{P}_{\hat{\rY}|\rS,\rY}(0|s,0)
    = \frac{1}{n_s^{-}} \mathsf{TN}_{s}.
\end{align*}
Hence, $\hat{\mathcal{C}}$ essentially characterizes all feasible confusion matrices and $\mathcal{C}$ is the population counterpart of $\hat{\mathcal{C}}$. Note that $\mathcal{C}$ is determined by the underlying data distribution while $\hat{\mathcal{C}}$ (and confusion matrices) are tailored to a specific dataset. 
\end{rem}

\begin{prop}
\label{prop::FATO_linear}
$\mathsf{FairFront}(\alpha_{\scalebox{.5}{\textnormal SP}}, \alpha_{\scalebox{.5}{\textnormal EO}}, \alpha_{\scalebox{.5}{\textnormal OAE}})$ in \eqref{eq::FAO} is equal to the solution of the following convex optimization:
\begin{subequations}
\label{eq::FATO_conv}
\begin{align}
    \max_{\bP\in \Reals^{AC\times C}}~&\sum_{s=1}^A\sum_{y=1}^C \mu_{s,y} P_{(s,y),y}\\
    \sto~&\mathsf{SP} \leq \alpha_{\scalebox{.5}{\textnormal SP}}, \mathsf{EO} \leq \alpha_{\scalebox{.5}{\textnormal EO}}, \mathsf{OAE} \leq \alpha_{\scalebox{.5}{\textnormal OAE}}\\
    &\bP \in \mathcal{C}. \label{eq::FATO_PinC}
\end{align}
\end{subequations}
Here the constants $\mu_{s,y} \defined \Pr(\rS=s,\rY=y)$ and $\mu_s \defined \Pr(\rS=s)$ for $s\in [A]$, $y\in[A]$ and $P_{(s,y),\hat{y}}$ denotes the $(C(s-1)+y)$-th row, $\hat{y}$-th column of the transition matrix $\bP$, which is $P_{\hat{\rY}|\rS,\rY}(\hat{y}|s,y)$.
\end{prop}

For example, in binary classification with a binary group attribute, the above optimization only has $8$ variables, $14$ linear constraints + a single convex constraint $\bP \in \mathcal{C}$. Hence, standard convex optimization solvers can directly compute its optimal value as long as we know how to characterize $\mathcal{C}$. 
\begin{rem}
\label{rem::ex_ind}
Note that \citet{kim2020fact} investigated fairness Pareto frontiers via confusion matrices. The main difference is that Definition~1 in \citet{kim2020fact} relaxed the constraint \eqref{eq::FATO_PinC} to $\bP \in \mathcal{T}(C|AC)$ where $\mathcal{T}(C|AC)$ represents \emph{all} transition matrices from $[AC]$ to $[C]$. This leads to a loose approximation of the frontier because $\mathcal{C}$ is often a strict subset of $\mathcal{T}(C|AC)$. To demonstrate this point, consider the scenario where $\rX 
\indep (\rS,\rY)$. Then $\hat{\rY} \indep (\rS,\rY)$ by data processing inequality so 
\begin{align}
\label{eq::XSY_ind_C}
    \mathcal{C} = \left\{\bP\in \mathcal{T}(C|AC) \mid \text{each row of }\bP \text{ is the same}\right\}.
\end{align}
Optimizing over $\mathcal{C}$ rather than $\mathcal{T}(C|AC)$ can significantly tighten the fairness Pareto frontier.
\end{rem}

Before diving into the analysis, we first introduce a function $\bg: \mathcal{X}\to \Delta_{AC}$ defined as $\bg(x) = \left(P_{\rS,\rY|\rX}(1,1|x), \cdots, P_{\rS,\rY|\rX}(A,C|x)\right).$ To obtain this function in practice, a common strategy among various post-processing fairness interventions \citep[see e.g.,][]{menon2018cost,alghamdi2022beyond} is to train a probabilistic classifier that uses input features $\rX$ to predict $(\rS,\rY)$. The output probability generated by this classifier is then utilized as an approximation of the function $\bg$.

The following theorem is the main theoretical result in this paper. It provides a precise characterization of the set $\mathcal{C}$ through a series of convex constraints. 
\begin{thm}
\label{thm::chara_FAT_set}
The set $\mathcal{C}$ is the collection of all transition matrices $\bP \in \mathcal{T}(C|AC)$ such that the following condition holds:\\
For any $k \in \mathcal{N}$ and any $\{\ba_i \mid \ba_i \in [-1,1]^{AC}, i\in[k]\}$,
\begin{align}
\label{eq::chara_C}
    \sum_{\hat{y}=1}^C \max_{i\in[k]} \left\{\ba_i^T\bLambda_{\mu} \bp_{\hat{y}}\right\}
    \leq \EE{\max_{i\in [k]}\{\ba_i^T \bg(\rX)\}},
\end{align}
where $\bp_{\hat{y}}$ is the $\hat{y}$-th column of $\bP$ and $\bLambda_{\mu}=\diag(\mu_{1,1},\cdots,\mu_{A,C})$.
\end{thm}

Intuitively, \eqref{eq::chara_C} uses piecewise linear functions to approximate the boundary of the convex set $\mathcal{C}$ where $k$ represents the number of linear pieces. Unfortunately, replacing $\bP \in \mathcal{C}$ with this series of constraints in \eqref{eq::FATO_conv} may result in an intractable problem since standard duality-based approaches will lead to infinitely many dual variables. To resolve this issue, we first fix $k$ and let $\mathcal{C}_k$ be the set of $\bP$ such that \eqref{eq::chara_C} holds under this fixed $k$. Accordingly, we define $\mathsf{FairFront}_k(\alpha_{\scalebox{.5}{\textnormal SP}}, \alpha_{\scalebox{.5}{\textnormal EO}}, \alpha_{\scalebox{.5}{\textnormal OAE}})$ as the optimal value of $\eqref{eq::FATO_conv}$ when replacing $\mathcal{C}$ with $\mathcal{C}_k$. Since $\mathcal{C}_1 \supseteq \mathcal{C}_2 \supseteq \cdots \supseteq \mathcal{C}$, we have $\mathsf{FairFront}_1(\alpha_{\scalebox{.5}{\textnormal SP}}, \alpha_{\scalebox{.5}{\textnormal EO}}, \alpha_{\scalebox{.5}{\textnormal OAE}}) \geq \mathsf{FairFront}_2(\alpha_{\scalebox{.5}{\textnormal SP}}, \alpha_{\scalebox{.5}{\textnormal EO}}, \alpha_{\scalebox{.5}{\textnormal OAE}})\geq \cdots \geq \mathsf{FairFront}(\alpha_{\scalebox{.5}{\textnormal SP}}, \alpha_{\scalebox{.5}{\textnormal EO}}, \alpha_{\scalebox{.5}{\textnormal OAE}}).$ However, computing $\mathsf{FairFront}_k(\alpha_{\scalebox{.5}{\textnormal SP}}, \alpha_{\scalebox{.5}{\textnormal EO}}, \alpha_{\scalebox{.5}{\textnormal OAE}})$ still involves infinitely many constraints.

\begin{algorithm}[!tb]
\begingroup
\small
\caption{Approximate the fairness Pareto frontier.}
\label{alg:FATO}

\begin{algorithmic}[*]
\State {\bfseries Input:} $\mathcal{D}=\{(x_i,y_i,s_i)\}_{i=1}^N$, max number of iterations $T$; max pieces $k$;  classifier $g(x)$; $\alpha_{\scalebox{.5}{\textnormal SP}}$, $\alpha_{\scalebox{.5}{\textnormal EO}}$, $\alpha_{\scalebox{.5}{\textnormal OAE}}$.

\State \textbf{Initialize:} \vspace{0.25em}
set $\mathcal{A} = \emptyset$; $\mu_{s,y} = \frac{\left|\{i|s_i=s,y_i=y\}\right|}{N}$; $t=1$.

\State \textbf{Repeat:}

\State \quad Solve a convex program: 
    \begin{subequations}
    \begin{align*}
        \max_{\bP}~&\sum_{s=1}^A\sum_{y=1}^C \mu_{s,y} P_{(s,y),y}\\
        \sto~&\bP \in \mathcal{T}(C|AC), 
        \mathsf{SP} \leq \alpha_{\scalebox{.5}{\textnormal SP}}, \mathsf{EO} \leq \alpha_{\scalebox{.5}{\textnormal EO}}, \mathsf{OAE} \leq \alpha_{\scalebox{.5}{\textnormal OAE}}\\
        &\sum_{\hat{y}=1}^C \max_{i\in[k]} \left\{\ba_i^T\bLambda_{\mu} \bp_{\hat{y}}\right\}
        \leq \EE{\max_{i\in [k]}\{\ba_i^T \bg(\rX)\}}
        \quad \forall (\ba_1,\cdots,\ba_k)\in \mathcal{A}.
    \end{align*}
    \end{subequations}
\State \quad Let $v^t$ and $\bP^t$ be the optimal value and optimal solution.

\State \quad Solve a DC program:
\begin{align*}
    \min_{\substack{\ba_i \in [-1,1]^{AC}\\ i\in[k]}}~\EE{\max_{i\in [k]}\{\ba_i^T \bg(\rX)\}} - \sum_{\hat{y}=1}^C \max_{i\in[k]}\left\{\ba_i^T\bLambda_{\mu} \bp_{\hat{y}}^t\right\}.
\end{align*}

\State \quad\textbf{If} the optimal value is $\geq 0$ or $t = T$, 

\State \quad \quad \textbf{stop};

\State \quad\textbf{otherwise}, 

\State \quad \quad add the optimal $(\ba_1,\cdots,\ba_k)$ to $\mathcal{A}$ and $t=t+1$.

\State \textbf{return:} $v^t$, $\bP^t$, $\mathcal{A}$.

\end{algorithmic}

\endgroup

\end{algorithm}

Next, we introduce a greedy improvement algorithm that consists of solving a sequence of tractable optimization problems for approximating $\mathsf{FairFront}_k(\alpha_{\scalebox{.5}{\textnormal SP}}, \alpha_{\scalebox{.5}{\textnormal EO}}, \alpha_{\scalebox{.5}{\textnormal OAE}})$. We use $\mathcal{A}$ to collect the constraints of $\bP$ and set $\mathcal{A}=\emptyset$ initially. At each iteration, our algorithm solves a convex program to find an optimal $\bP$ that maximizes the accuracy while satisfying the desired group fairness constraints and the constraints in $\mathcal{A}$; then we verify if this $\bP$ is within the set $\mathcal{C}_k$ by solving a DC (difference of convex) program \citep{shen2016disciplined,horst1999dc}. If $\bP \in \mathcal{C}_k$, then the algorithm stops. Otherwise, the algorithm will find the constraint that is mostly violated by $\bP$ and add this constraint to $\mathcal{A}$. 
Specifically, we determine a piecewise linear function that divides the space into two distinct regions: one containing $\bP$ and the other containing $\mathcal{C}_k$. By ``mostly violated'', we mean the function is constructed to maximize the distance between $\bP$ and the boundary defined by the function. 
We describe our algorithm in Algorithm~\ref{alg:FATO} and establish a convergence guarantee below.

\begin{thm}
\label{thm::convergence}
Let $T=\infty$. If Algorithm~\ref{alg:FATO} stops, its output $\bP^t$ is an optimal solution of $\mathsf{FairFront}_k(\alpha_{\scalebox{.5}{\textnormal SP}}, \alpha_{\scalebox{.5}{\textnormal EO}}, \alpha_{\scalebox{.5}{\textnormal OAE}})$. Otherwise, any convergent sub-sequence of $\{\bP^t\}_{t=1}^{\infty}$ converges to an optimal solution of $\mathsf{FairFront}_k(\alpha_{\scalebox{.5}{\textnormal SP}}, \alpha_{\scalebox{.5}{\textnormal EO}}, \alpha_{\scalebox{.5}{\textnormal OAE}})$.
\end{thm}
Note that the output $v^t$ from Algorithm~\ref{alg:FATO} is always an \emph{upper bound} for $\mathsf{FairFront}(\alpha_{\scalebox{.5}{\textnormal SP}}, \alpha_{\scalebox{.5}{\textnormal EO}}, \alpha_{\scalebox{.5}{\textnormal OAE}})$, assuming the estimation error is sufficiently small. The tightness of this upper bound is determined by $k$ (i.e., how well the piecewise linear functions approximate the boundary of $\mathcal{C}$), $T$ (i.e., the total number of iterations). On the other hand, running off-the-shelf in-processing and post-processing fairness interventions can only yield \emph{lower bounds} for $\mathsf{FairFront}(\alpha_{\scalebox{.5}{\textnormal SP}}, \alpha_{\scalebox{.5}{\textnormal EO}}, \alpha_{\scalebox{.5}{\textnormal OAE}})$.

\section{Numerical Experiments}

In this section, we demonstrate the tightness of our upper bound approximation of $\mathsf{FairFront}$, apply it to benchmark existing group fairness interventions, and show how data biases, specifically missing values, impact their effectiveness. 
We find that given sufficient data, SOTA fairness interventions are successful at reducing epistemic discrimination as their gap to (our upper bound estimate of) $\mathsf{FairFront}$ is small. 
However, we also discover that when different population groups have varying missing data patterns, aleatoric discrimination increases, which diminishes the performance of fairness intervention algorithms.
Our numerical experiments are semi-synthetic since we apply fairness interventions to train classifiers using the \emph{entire} dataset and resample from it as the test set. This setup enables us to eliminate the estimation error associated with Algorithm~\ref{alg:FATO} (see Appendix~\ref{append:future} for a discussion). 
We provide additional experimental results and details in Appendix~\ref{sec::append_exp}.

\subsection{Benchmark Fairness Interventions}
\label{subsec::bench_fair_interv}

\paragraph{Setup.} We evaluate our results on the UCI Adult dataset \citep{bache2013uci}, the ProPublica COMPAS dataset \citep{angwin2016machine}, the German Credit dataset \citep{bache2013uci}, and HSLS (High School Longitudinal Study) dataset \citep{ingels2011high,jeong2022fairness}. We recognize that Adult, COMPAS, and German Credit datasets are overused and acknowledge the recent calls to move away from them \citep[see e.g.,][]{ding2021retiring}. We adopt these datasets for benchmarking purposes only since most fairness interventions have available code for these datasets. 
The HSLS dataset is a new dataset that first appeared in the fair ML literature last year and captures a common use-case of ML in education \citep[student performance prediction, see ][]{jeong2022fairness}. It has multi-class labels and multiple protected groups.
We apply existing (group) fairness interventions to these datasets and measure their fairness violations via \emph{Max equalized odds}:
\begin{equation*}
    \max~|\Pr(\hat{\rY} = \hat{y} | \rS=s,\rY=y) - \Pr(\hat{\rY} = \hat{y} | \rS=s',\rY=y)|
\end{equation*}
where the max is taken over $y,\hat{y},s,s'$. We run Algorithm~\ref{alg:FATO} with $k=6$ pieces, $20$ iterations, and varying $\alpha_{\scalebox{.5}{\textnormal EO}}$ to estimate $\mathsf{FairFront}$ on each dataset. We compute the expectations and the $g$ function from the empirical distributions and solve the DC program by using the package in \citet{shen2016disciplined}. 
The details about how we pre-process these datasets and additional experimental results on the German Credit and HSLS datasets are deferred to Appendix~\ref{sec::append_exp}.

\paragraph{Group fairness interventions.} We consider five existing fairness-intervention algorithms: \texttt{Reduction}~\citep{agarwal2018reductions},  \texttt{EqOdds}~\citep{hardt2016equality}, \texttt{CalEqOdds}~\citep{pleiss2017fairness}, \texttt{LevEqOpp}~\citep{chzhen2019leveraging}, and \texttt{FairProjection}~\cite{alghamdi2022beyond}. Among them, \texttt{Reduction} is an in-processing method and the rest are all post-processing methods. For the first three benchmarks, we use the implementations from IBM AIF360 library~\citep{bellamy2018ai}; for \texttt{LevEqOpp} and \texttt{FairProjection}, we use the Python implementations from the Github repo in \citet{alghamdi2022beyond}. For \texttt{Reduction} and \texttt{FairProjection}, we can vary their tolerance of fairness violations to produce a fairness-accuracy curve; for \texttt{EqOdds}, \texttt{CalEqOdds}, and \texttt{LevEqOpp}, each of them produces a single point since they only allow hard equality constraint. We note that \texttt{FairProjection} is optimized for transforming probabilistic classifier outputs \citep[see also][]{wei2021optimized}, but here we threshold the probabilistic outputs to generate binary predictions which may limit its performance. Finally, we train a random forest as the \texttt{Baseline} classifier. 

\begin{figure*}[t]
\centering
\includegraphics[width=0.42\linewidth]{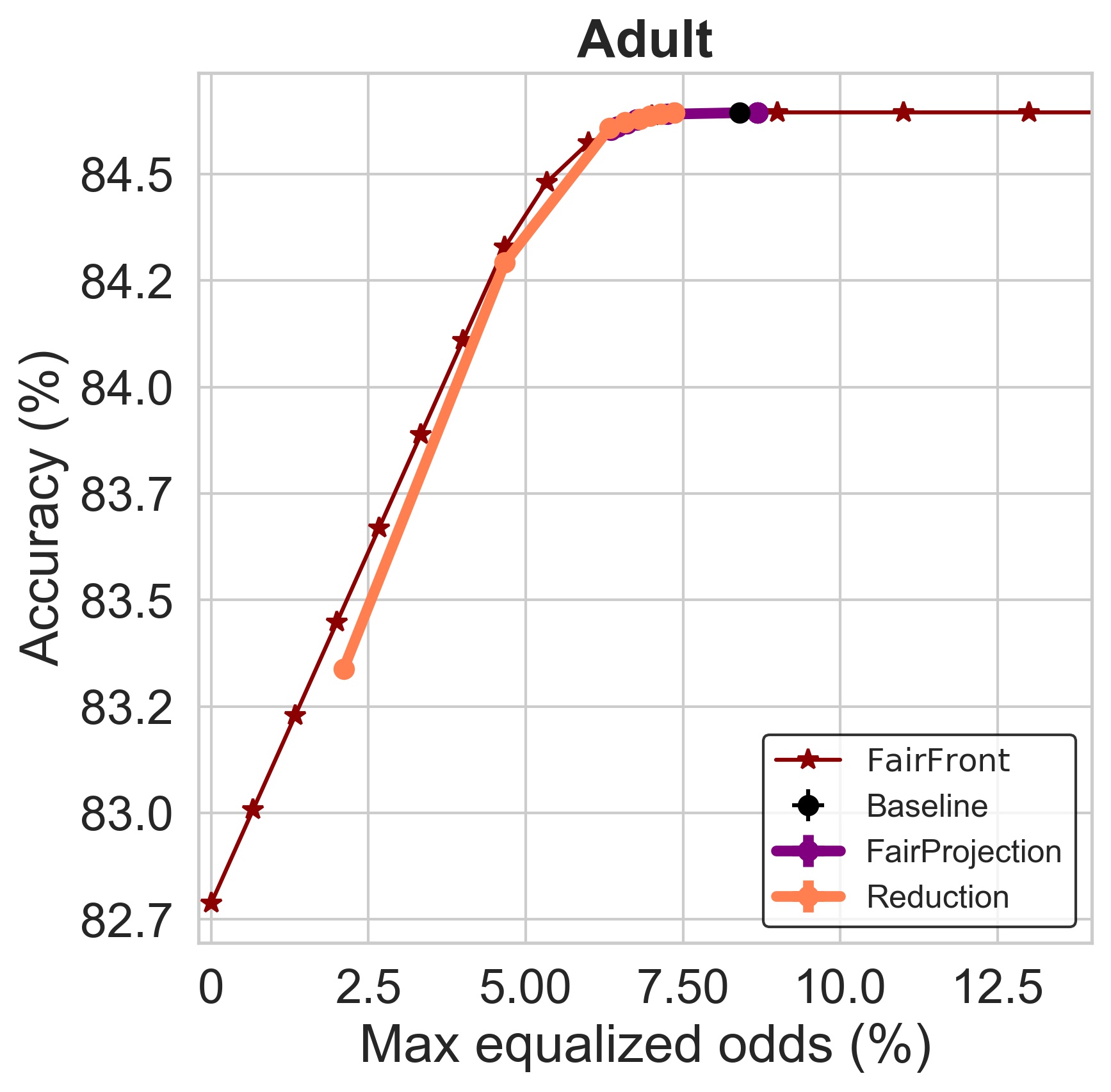}
\hspace{0.5cm}
\includegraphics[width=0.42\linewidth]{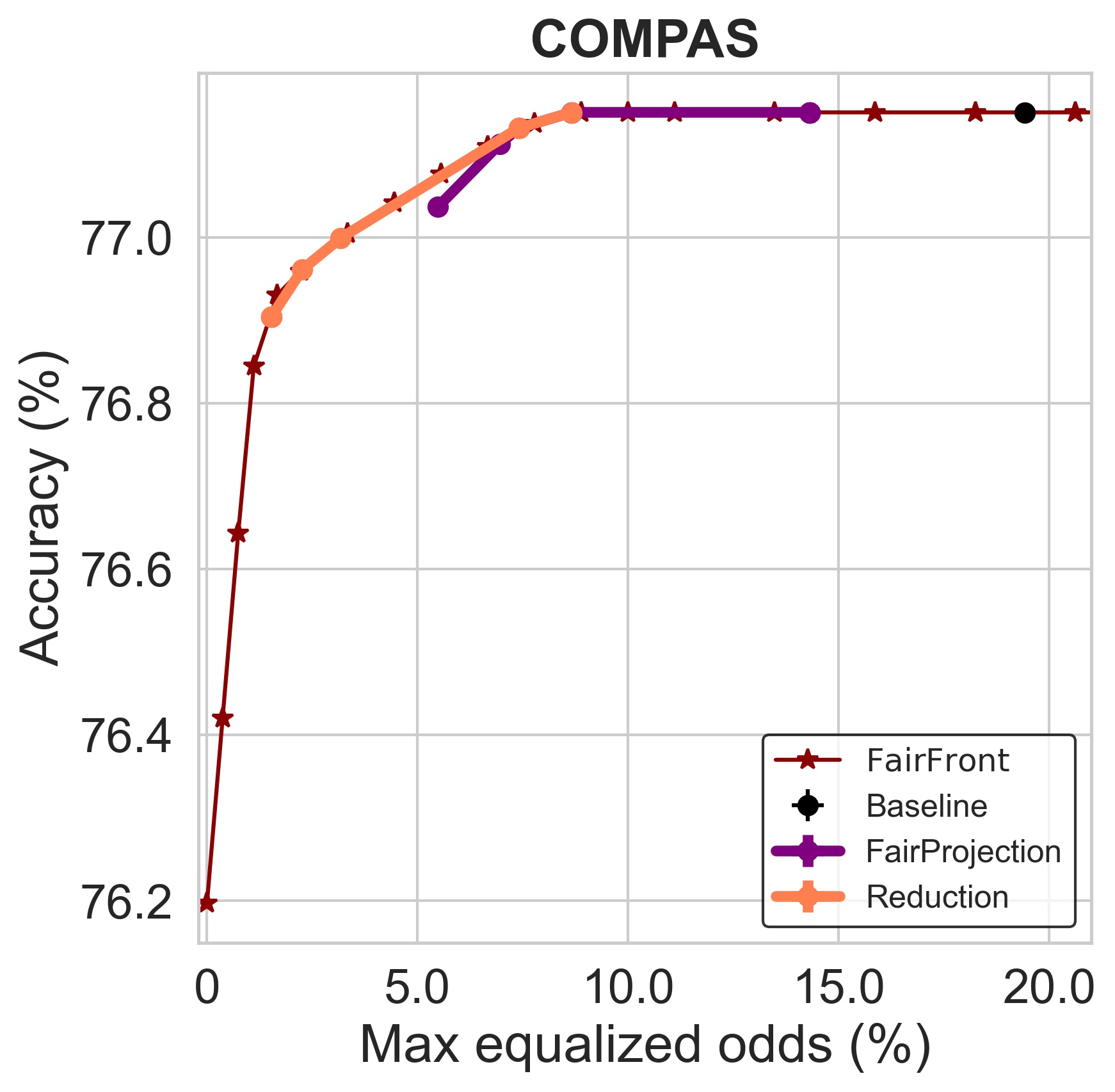}
\caption{We compare \texttt{Reduction} and \texttt{FairProjection} with (our upper bound estimate of) $\mathsf{FairFront}$ on the Adult (Left) and COMPAS (Right) datasets. We train a classifier that approximates the Bayes optimal and use it as a basis for \texttt{Reduction} and \texttt{FairProjection}. This result not only demonstrates the tightness of our approximation but also shows that SOTA fairness interventions have already achieved near-optimal fairness-accuracy curves.
}
\label{Fig::BayesOptimal}
\end{figure*}

\begin{figure*}[t]
\centering
\includegraphics[width=0.42\linewidth]{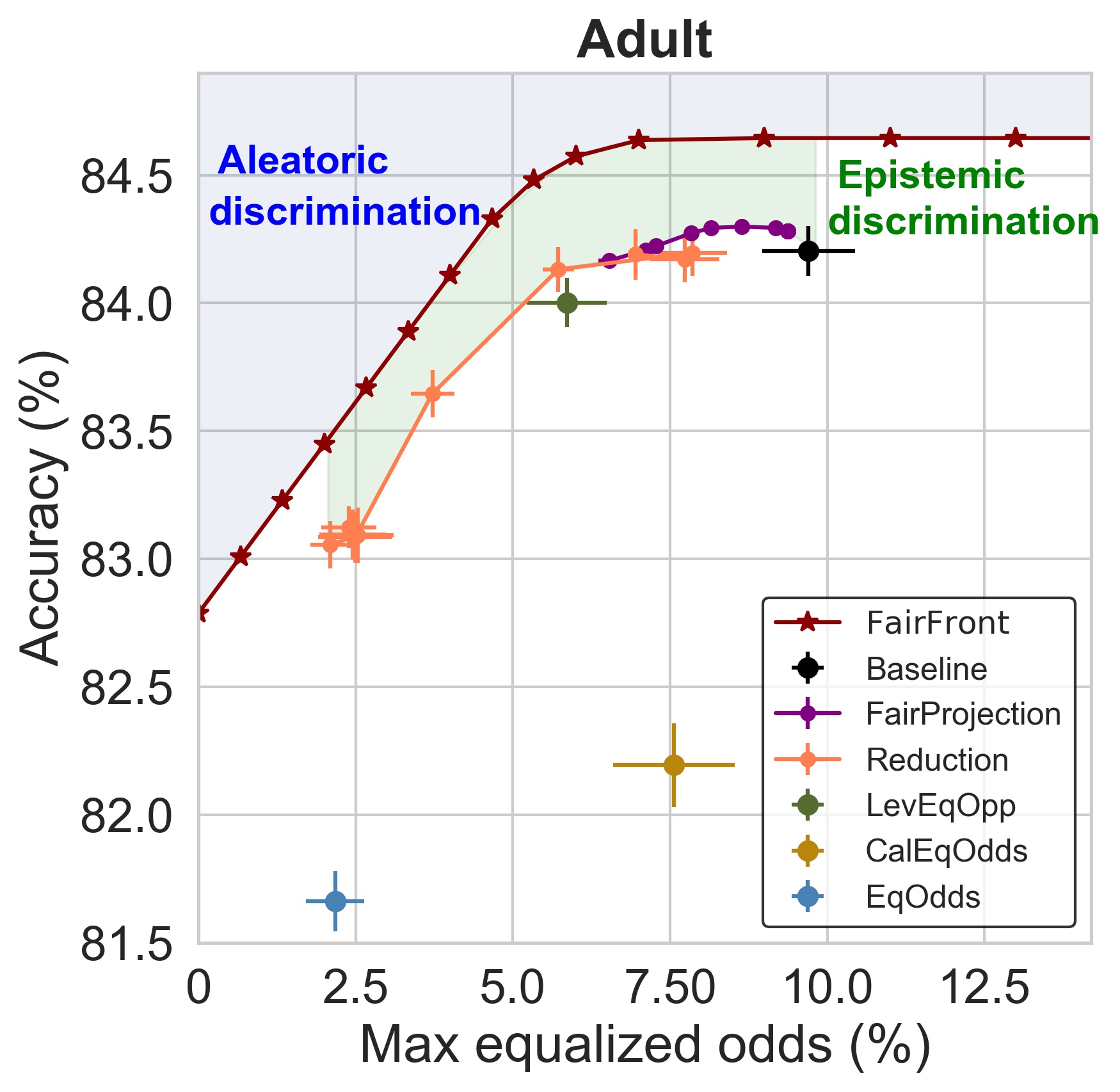}
\hspace{0.5cm}
\includegraphics[width=0.42\linewidth]{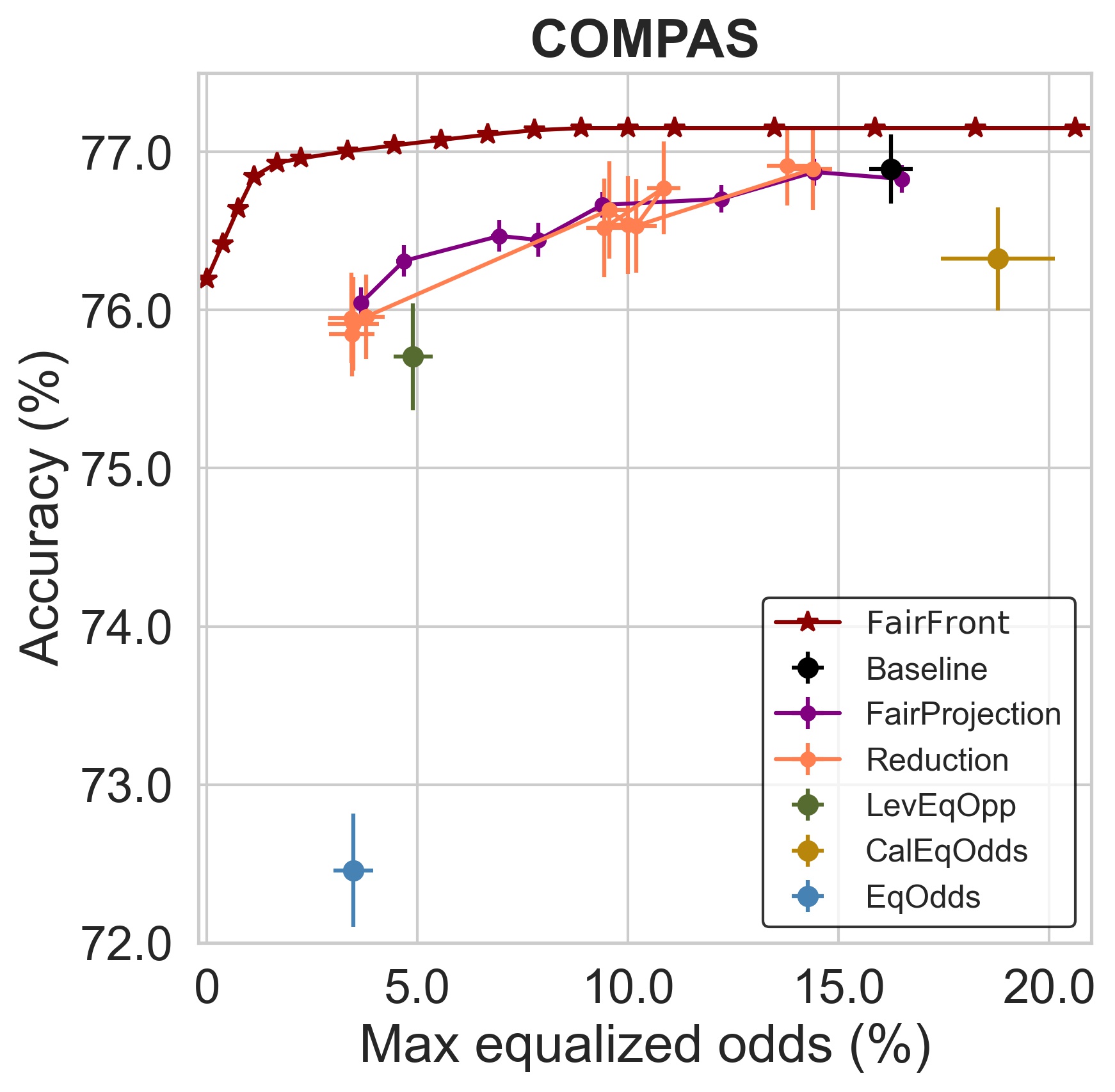}
\caption{We benchmark existing fairness interventions using (our upper bound estimate of) $\mathsf{FairFront}$. 
We use $\mathsf{FairFront}$ to quantify aleatoric discrimination and measure epistemic discrimination by comparing a classifier's accuracy and fairness violation with $\mathsf{FairFront}$. The results show that SOTA fairness interventions are effective at reducing epistemic discrimination.
}
\label{Fig::Fair_Frontier}
\end{figure*}

\paragraph{Results.} We observe that if we run Algorithm~\ref{alg:FATO} for a single iteration, which is equivalent to solving Proposition~\ref{prop::FATO_linear} without \eqref{eq::FATO_PinC}, its solution is very close to $1$ for all $\alpha_{\scalebox{.5}{\textnormal EO}}$. This demonstrates the benefits of incorporating Blackwell's conditions into the fairness Pareto frontier.

We train a classifier that approximates the Bayes optimal and use it as a basis for both \texttt{Reduction} and \texttt{FairProjection}, which are SOTA fairness interventions. We then apply these two fairness interventions to the entire dataset and evaluate their performance on the same dataset. Figure~\ref{Fig::BayesOptimal} shows that in this infinite sample regime, the fairness-accuracy curves produced by \texttt{Reduction} and \texttt{FairProjection} can approach our upper bound estimate of $\mathsf{FairFront}$. This result not only demonstrates the tightness of our approximation (recall that Algorithm~\ref{alg:FATO} gives an upper bound of $\mathsf{FairFront}$ and existing fairness interventions give  lower bounds) but also shows that SOTA fairness interventions have already achieved near-optimal fairness-accuracy curves.

Recall that we use $\mathsf{FairFront}$ to quantify aleatoric discrimination since it characterizes the highest achievable accuracy among all classifiers satisfying the desired fairness constraints. Additionally, we measure epistemic discrimination by comparing a classifier's accuracy and fairness violation with $\mathsf{FairFront}$. 
Given that our Algorithm~\ref{alg:FATO} provides a tight approximation of $\mathsf{FairFront}$, we use it to benchmark existing fairness interventions. 
Specifically, we first train a base classifier which may not achieve Bayes optimal accuracy. Then we use it as a basis for all existing fairness interventions. 
The results in Figure~\ref{Fig::Fair_Frontier} show that SOTA fairness interventions remain effective at reducing epistemic discrimination. 
In what follows, we demonstrate how missing values in data can increase aleatoric discrimination and dramatically reduce the effectiveness of SOTA fairness interventions.

\subsection{Fairness Risks in Missing Values}

Real-world data often have missing values and the missing patterns can be different across different protected groups \citep[see][for some examples]{jeong2022fairness}. There is a growing line of research \citep[see e.g.,][]{jeong2022fairness,fernando2021missing,wang2021analyzing,subramoniandiscrimination,caton2022impact,zhang2021assessing,schelter2019fairprep} studying the fairness risks of data with missing values. 
In this section, we apply our result to demonstrate how disparate missing patterns influence the fairness-accuracy curves.

\paragraph{Setup.} 

We choose sex (group~0: female, group~1: male) as the group attribute for the Adult dataset, and race (group~0: African-American, group~1: Caucasian) for the COMPAS dataset. To investigate the impact of disparate missing patterns on aleatoric discrimination, we artificially generate missing values in both datasets. This is necessary as the datasets do not contain sufficient missing data. The missing values are generated according to different probabilities for different population groups. For each data point from group~0, we erase each input feature with a varying probability $p_0 \in \{10\%, 50\%, 70\%\}$, while for group~1, we erase each input feature with a fixed probability $p_1 = 10\%$. We then apply mode imputation to the missing values, replacing them with the mode of non-missing values for each feature. Finally, we apply Algorithm~\ref{alg:FATO} along with \texttt{Reduction} and \texttt{Baseline} to the imputed data. The experimental results are shown in Figure~\ref{Fig::Reduce_Aleatoric}.

\begin{figure*}[t]
\centering
\includegraphics[width=0.42\linewidth]{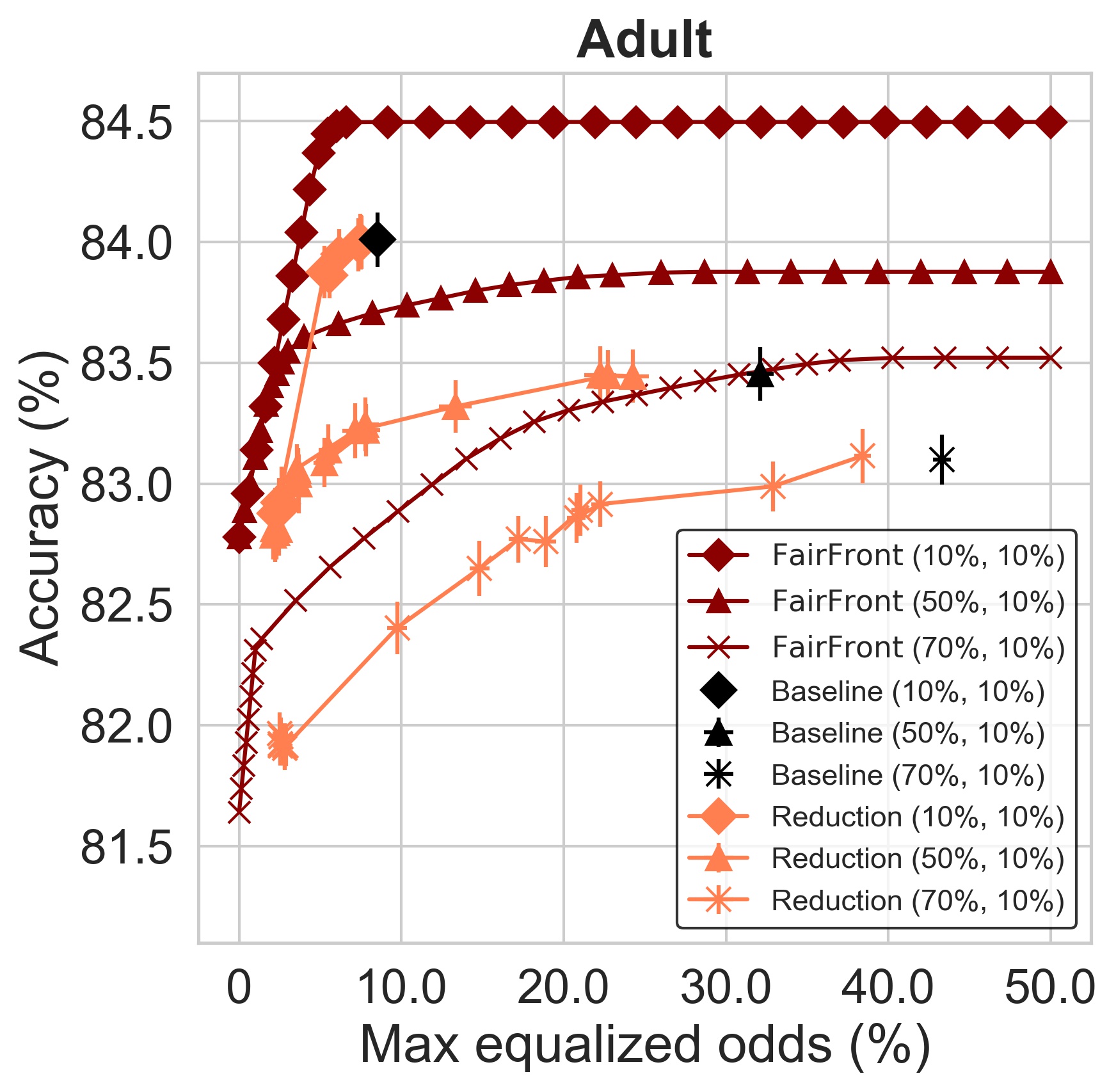}
\hspace{0.5cm}
\includegraphics[width=0.42\linewidth]{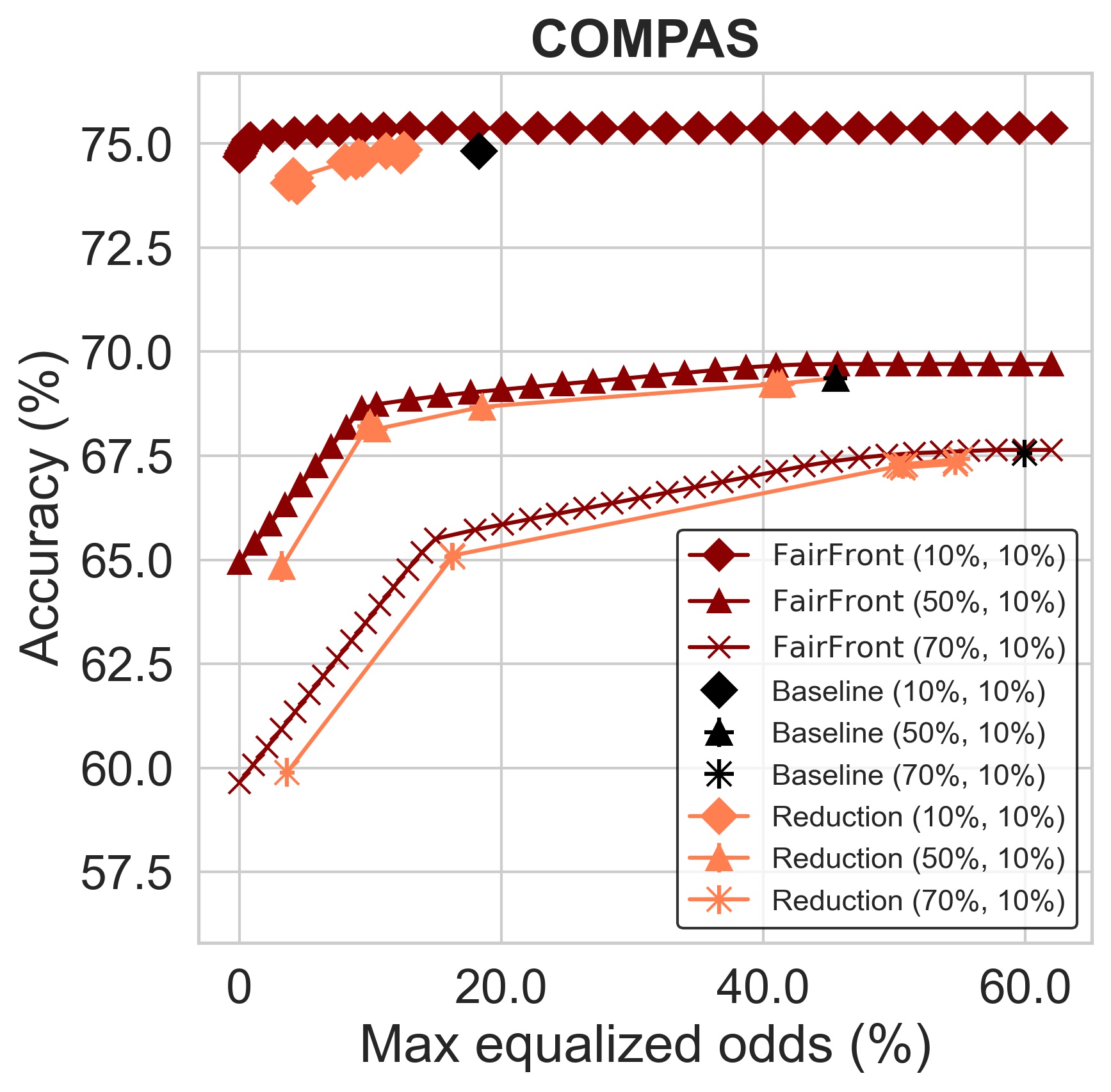}
\caption{Fairness risks of disparate missing patterns. The missing probabilities of group 0 (female in Adult/African-American in COMPAS) and group 1 (male in Adult/Caucasian in COMPAS) are varying among $\{(10\%, 10\%), (50\%, 10\%), (70\%, 10\%)\}$. We apply \texttt{Reduction} and \texttt{Baseline} to the imputed data and plot their fairness-accuracy curves against $\mathsf{FairFront}$. As shown, the effectiveness of fairness interventions substantially decrease with increasing disparate missing patterns in data. 
}
\label{Fig::Reduce_Aleatoric}
\end{figure*}

\paragraph{Results.} As we increase the missing probability of group~0, (our upper bound estimate of) $\mathsf{FairFront}$ decreases since it becomes more difficult to accurately predict outcomes for group 0. This in turn affects the overall model performance, since the fairness constraint requires that the model performs similarly for both groups. We also observe the fairness-accuracy curves of \texttt{Reduction} decrease as the missing data for group 0 become more prevalent. In other words, as the missing data for group~0 increase, it becomes more difficult to maintain both high accuracy and fairness in the model's prediction.

\section{Final Remarks}

The past years have witnessed a growing line of research introducing various group fairness-intervention algorithms. Most of these interventions focus on optimizing model performance subject to group fairness constraints. Though comparing and benchmarking these methods on various datasets is valuable \citep[e.g., see benchmarks in][]{friedler2019comparative,bellamy2019ai,wei2021optimized}, this does not reveal if there is still room for improvement in their fairness-accuracy curves, or if existing methods approach the information-theoretic optimal limit when infinite data is available. Our results address this gap by introducing the fairness Pareto frontier, which measures the highest possible accuracy under a set of group fairness constraints. We precisely characterize the fairness Pareto frontier using Blackwell's conditions and present a greedy improvement algorithm that approximates it from data. Our results show that the fairness-accuracy curves produced by SOTA fairness interventions are very close to the fairness Pareto frontier on standard datasets.

Additionally, we demonstrate that when data are biased due to missing values, the fairness Pareto frontier degrades. Although existing fairness interventions can still reduce performance disparities, they come at the cost of significantly lowering overall model accuracy. The methods we present for computing the fairness Pareto frontier can also be applied to analyze other sources of aleatoric discrimination, such as when individuals may misreport their data or when there are measurement errors. Overall, the fairness Pareto frontier can serve as a valuable framework for guiding data collection and cleaning.
Our results indicate that existing fairness interventions can be effective in reducing epistemic discrimination, and there are diminishing returns in developing new fairness interventions focused solely on optimizing accuracy for a given group fairness constraint on pristine data. However, existing fairness interventions have yet to effectively provide both fair and accurate classification when additional sources of aleatoric discrimination are present (such as missing values in data). This suggests that there is still significant need for research on handling aleatoric sources of discrimination that appear throughout the data collection process.

We provide an in-depth discussion on future work in Appendix~\ref{append:future}.

%% file: Appendix.tex
\section{Technical Background}

In this section, we extend some results in \citet{blackwell1951comparison,blackwell1953equivalent} to our setting. 
For a random variable $\rX$, we denote its probability distribution by $\mathcal{L}(\rX)$. 
A conditional distribution $P_{\rX|\rA}: [n] \to \mathcal{P}(\mathcal{X})$ can be equivalently written as $\bP \defined (P_1,\cdots,P_n)$ where each $P_i = P_{\rX|\rA=i} \in \mathcal{P}(\mathcal{X})$. Let $\mathcal{A}$ be a closed, bounded, convex subset of $\Reals^n$. A decision function is a mapping $\bm{f}: \mathcal{X} \to \mathcal{A}$, which can also be written as $\bm{f}(x) = (a_1(x),\cdots, a_n(x))$. A decision function is associated a loss vector:
\begin{align}
    \bv(\bm{f}) 
    = \left(\int a_1(x) \dif P_1(x),\cdots, \int a_n(x) \dif P_n(x) \right).
\end{align}
The collection of all $\bv(\bm{f})$ is denoted by $\mathcal{B}(P_{\rX|\rA},\mathcal{A})$ or $\mathcal{B}(\bP,\mathcal{A})$.

For a vector $\blambda \in \Delta_n$ such that $\blambda > 0$, we define a function $\bp_{\blambda}(x): \mathcal{X} \to \Delta_n$:
\begin{align}
\label{eq::disc_func}
    \bp_{\blambda}(x) = \left(\frac{\lambda_1\dif P_1}{\lambda_1\dif P_1 + \cdots + \lambda_n \dif P_n} ,\cdots, \frac{\lambda_n\dif P_n}{\lambda_1 \dif P_1 + \cdots + \lambda_n \dif P_n}\right).
\end{align}
Note that $\bp_{\blambda}(\rX)$ is a sufficient statistic for $\rX$, considering $\rA$ as the parameter (it can be proved by Fisher-Neyman factorization theorem). In other words, two Markov chains hold: $\rA - \bp_{\blambda}(\rX) - \rX$ and $\rA - \rX - \bp_{\blambda}(\rX)$ for any distribution on $\rA$.

Consider a new set of probability distributions $\bP^*_{\blambda} \defined (\mathcal{L}(\bp_{\blambda}(\rX_1)), \cdots, \mathcal{L}(\bp_{\blambda}(\rX_n)))$ where $\mathcal{L}(\rX_i) = P_i$. Here $\bP^*_{\blambda}$ can be viewed as a conditional distribution from $[n]$ to $\mathcal{P}(\Delta_n)$ since each $\mathcal{L}(\bp_{\blambda}(\rX_i))$ is a probability distribution over $\Delta_n$. The following lemma follows from the sufficiency of $\bp_{\blambda}(\rX)$. 
\begin{lem}[Adaptation of Theorem~3 in \citet{blackwell1951comparison}]
\label{lem::Bset_suff}
For any $\mathcal{A}$, $\mathcal{B}(\bP,\mathcal{A}) = \mathcal{B}(\bP^*_{\blambda},\mathcal{A})$.
\end{lem}
\begin{proof}
Suppose that $\bm{f}^*(\bp) = \left(a_1^*(\bp), \cdots, a_n^*(\bp)\right)$ is a decision function for $(\bP^*_{\blambda},\mathcal{A})$. Accordingly, we define $\bm{f}(x) = \left(a_1^*(\bp_{\blambda}(x)), \cdots, a_n(\bp_{\blambda}(x))\right)$ where the function $\bp_{\blambda}$ is defined in \eqref{eq::disc_func}. Then it is clear that $\bm{f}$ is a decision function for $(\bP,\mathcal{A})$. By the law of unconscious statistician, we have
\begin{align}
\label{eq::unc_stat}
\int a_i^*(\bp) \dif P_{_{\blambda},i}^*(\bp)
= \EE{a_i^*(\bp_{\blambda}(\rX_i))}
= \int a_i^*(\bp_{\blambda}(x)) \dif P_i(x).
\end{align}
Hence, $\bv(\bm{f}^*) = \bv(\bm{f})$, which implies $\mathcal{B}(\bP^*_{\blambda},\mathcal{A}) \subseteq \mathcal{B}(\bP,\mathcal{A})$. For the other direction, suppose $\bm{f}(x) = \left(a_1(x), \cdots, a_n(x)\right)$ is a decision function for $(\bP, \mathcal{A})$. Let $\bm{f}^*(\bp) = \left(a_1^*(\bp), \cdots, a_n^*(\bp)\right)$ where $a_i^*(\bp) \defined \EE{a_i(\rX_i)\mid \bp_{\blambda}(\rX_i) = \bp}$. Since $\bp_{\blambda}(\rX)$ is a sufficient statistics, for any $i\in[n]$
\begin{align}
    \mathcal{L}(\rX_i|\bp_{\blambda}(\rX_i)=\bp) = \mathcal{L}(\rX_1|\bp_{\blambda}(\rX_1)=\bp).
\end{align}
Therefore, $\bm{f}^*(\bp) = \EE{\bm{f}(\rX_1)|\bp_{\blambda}(\rX_1) = \bp}$. Since $\mathcal{A}$ is a convex set, $\bm{f}^*$ is a decision function for $(\bP^*,\mathcal{A})$.
By the law of total expectation, we have
\begin{align}
    \int a_i^*(\bp) \dif P_{\blambda,i}^*(\bp)
    = \int a_i(x) \dif P_i(x).
\end{align}
Hence, $\bv(\bm{f}) = \bv(\bm{f}^*)$, which implies $\mathcal{B}(\bP,\mathcal{A}) \subseteq \mathcal{B}(\bP^*_{\blambda},\mathcal{A})$.
\end{proof}

For a vector $\blambda \in \Delta_n$ such that $\blambda > 0$, the condition distribution $P_{\rX|\rA}$ induces a weighted standard measure $P^*_{\blambda} \defined \mathcal{L}\left(\bp_{\blambda}(\bar{\rX})\right)$ where $\mathcal{L}(\bar{\rX}) = \lambda_1 P_1 + \cdots + \lambda_n P_n$.

\begin{thm}[Adaptation of Theorem~4 in \citet{blackwell1951comparison}]
\label{thm::stand_meas}
For any two conditional distributions $P_{\rX|\rA}$ and $Q_{\rY|\rA}$, let $P^*_{\blambda}$ and $Q^*_{\blambda}$ be their weighted standard measures, respectively. Then $\mathcal{B}(P_{\rX|\rA},\mathcal{A}) \supseteq \mathcal{B}(Q_{\rY|\rA},\mathcal{A})$ for any closed, bounded, convex set $\mathcal{A}$ if and only if for any continuous convex $\phi:\Delta_n \to \Reals$, $\int \phi(\bp) \dif P^*_{\blambda}(\bp) \geq \int \phi(\bp) \dif Q^*_{\blambda}(\bp)$
\end{thm}

\begin{proof}
First, by Lemma~\ref{lem::Bset_suff}, we know
$\mathcal{B}(P_{\rX|\rA},\mathcal{A}) = \mathcal{B}(\bP^*_{\blambda},\mathcal{A})$ and $\mathcal{B}(Q_{\rY|\rA},\mathcal{A}) = \mathcal{B}(\bQ^*_{\blambda},\mathcal{A})$. We denote $\bLambda = \diag(\lambda_1,\cdots,\lambda_n)$. Consider any $\mathcal{A} = \mathsf{conv}(\ba_1, \cdots, \ba_k)$. Let 
\begin{align}
\label{eq::defn_fstar}
    \bm{f}^*(\bp) = \argmin_{\ba \in \mathcal{A}}\bp^T \bLambda^{-1} \ba.
\end{align}
Note that $\bm{f}^*(\bp) \in \{\ba_1,\cdots,\ba_k\}$ since this set contains all the extreme points of $\mathcal{A}$.\footnote{If \eqref{eq::defn_fstar} has multiple optimal solutions, we always choose the one from $\{\ba_1,\cdots,\ba_k\}$.} By definition, for any decision function w.r.t. $(\bP^*_{\blambda},\mathcal{A})$, we have
\begin{align}
    \bp^T \bLambda^{-1} \bm{f}(\bp) 
    \geq \bp^T \bLambda^{-1} \bm{f}^*(\bp),\quad \forall \bp.
\end{align}
Let $\bv = \bv(\bm{f})$. By the same reason with \eqref{eq::unc_stat}, we have
\begin{align}
    v_j 
    &= \int a_j(\bp_{\blambda}(x)) \dif P_j(x)\\
    &= \frac{1}{\lambda_j}\int a_j(\bp_{\blambda}(x)) \frac{\lambda_j\dif P_j}{\lambda_1\dif P_1 + \cdots + \lambda_n\dif P_n}(x) (\lambda_1 \dif P_1 + \cdots + \lambda_n\dif P_n)(x)\\
    &= \frac{1}{\lambda_j}\int a_j(\bp_{\blambda}(x)) [\bp_{\blambda}(x)]_j (\lambda_1 \dif P_1 + \cdots + \lambda_n\dif P_n)(x)\\
    &= \frac{1}{\lambda_j}\EE{a_j(\bp_{\blambda}(\bar{\rX})) [\bp_{\blambda}(\bar{\rX})]_j}\\
    &= \frac{1}{\lambda_j}\int a_j(\bp) p_j \dif P^*_{\blambda}(\bp),
\end{align}
where the last step is due to the law of unconscious statistician. Therefore,
\begin{align}
    \sum_{j=1}^n v_j
    &= \int \bp^T \bLambda^{-1} \bm{f}(\bp) \dif P^*_{\blambda}(\bp)\\
    &\geq \int \bp^T \bLambda^{-1} \bm{f}^*(\bp) \dif P^*_{\blambda}(\bp)\\
    &= \int \min_{i}\{ \bp^T \bLambda^{-1} \ba_i\} \dif P^*_{\blambda}(\bp).
\end{align}
The equality is achieved by $\bv(\bm{f}^*)$. Hence, for any $\mathcal{A} = \mathsf{conv}(\ba_1, \cdots, \ba_k)$
\begin{align}
\label{eq::min_v_rewrite}
    \min_{\bv \in \mathcal{B}(P_{\rX|\rA},\mathcal{A})} \sum_{j=1}^n v_j
    = \int \min_{i}\{\ba_i^T \bLambda^{-1} \bp\} \dif P^*_{\blambda}(\bp).
\end{align}
Recall that Theorem~2.(3) in \citet{blackwell1951comparison} states
\begin{align*}
    &\mathcal{B}(P_{\rX|\rA},\mathcal{A}) \supseteq \mathcal{B}(P_{\rY|\rA},\mathcal{A})\quad \text{for every closed, bounded, convex }\mathcal{A}\\
    \Leftrightarrow
    &\min_{\bv \in \mathcal{B}(P_{\rX|\rA},\mathcal{A})} \sum_{j=1}^n v_j
    \leq \min_{\bv \in \mathcal{B}(P_{\rY|\rA},\mathcal{A})} \sum_{j=1}^n v_j \quad \text{for every closed, bounded, convex }\mathcal{A}.
\end{align*}
By approximation theory, the second condition can be relaxed to any $\mathcal{A}$ that is a convex hull of a finite set. By \eqref{eq::min_v_rewrite}, this relaxed condition is equivalent to 
\begin{align}
    \int \phi(\bp) \dif P^*_{\blambda}(\bp)
    \geq \int \phi(\bp) \dif Q^*_{\blambda}(\bp)
\end{align}
for all $\phi(\bp)$ that are the \textbf{maximum} of finitely many linear functions. By approximation theory again, the above condition is equivalent to the one holding for any continuous convex function $\phi$.
\end{proof}

\section{Omitted Proofs}

\subsection{Proof of Lemma~\ref{lem::prop_conv_C}}

\begin{proof}
Clearly, $\mathcal{C}$ is a subset of $\mathcal{T}(C|AC)$. Let $\lambda \in (0,1)$ and $P_{\hat{\rY}_0|\rS,\rY}, P_{\hat{\rY}_1|\rS,\rY} \in \mathcal{C}$. Now we introduce a Bernoulli random variable $\rB$ such that $\Pr(\rB = 0) = \lambda$. Finally, we define $\hat{\rY}_{\lambda} = \rB \hat{\rY}_1 + (1-\rB) \hat{\rY}_0$. By definition, we have $(\rS,\rY) - \rX - \hat{\rY}_{\lambda}$ so $P_{\hat{\rY}_\lambda|\rS,\rY} \in \mathcal{C}$. Moreover, 
\begin{align*}
    P_{\hat{\rY}_\lambda|\rS,\rY}
    = \lambda P_{\hat{\rY}_0|\rS,\rY} + (1-\lambda) P_{\hat{\rY}_1|\rS,\rY}.
\end{align*}
Hence, $\mathcal{C}$ is convex.

Let $\lambda \in (0,1)$. Assume $\bP$ and $\bar{\bP}$ achieve the maximal values of Proposition~\ref{prop::FATO_linear} under $(\alpha_{\scalebox{.5}{\textnormal SP}}, \alpha_{\scalebox{.5}{\textnormal EO}}, \alpha_{\scalebox{.5}{\textnormal OAE}})$ and $(\bar{\alpha}_{\scalebox{.5}{\textnormal SP}}, \bar{\alpha}_{\scalebox{.5}{\textnormal EO}}, \bar{\alpha}_{\scalebox{.5}{\textnormal OAE}})$, respectively. We define $\bP_{\lambda} = \lambda \bP + (1-\lambda) \bar{\bP}$, which satisfies the constraints of Proposition~\ref{prop::FATO_linear} with thresholds $(\lambda \alpha_{\scalebox{.5}{\textnormal SP}} + (1-\lambda)\bar{\alpha}_{\scalebox{.5}{\textnormal SP}}, \lambda \alpha_{\scalebox{.5}{\textnormal EO}} + (1-\lambda)\bar{\alpha}_{\scalebox{.5}{\textnormal EO}}, \lambda \alpha_{\scalebox{.5}{\textnormal OAE}} + (1-\lambda)\bar{\alpha}_{\scalebox{.5}{\textnormal OAE}})$. Finally, since the objective function of Proposition~\ref{prop::FATO_linear} is a linear function, it is equal to $\lambda \mathsf{FairFront}(\alpha_{\scalebox{.5}{\textnormal SP}}, \alpha_{\scalebox{.5}{\textnormal EO}}, \alpha_{\scalebox{.5}{\textnormal OAE}}) + (1-\lambda) \mathsf{FairFront}(\bar{\alpha}_{\scalebox{.5}{\textnormal SP}}, \bar{\alpha}_{\scalebox{.5}{\textnormal EO}}, \bar{\alpha}_{\scalebox{.5}{\textnormal OAE}})$ under $\bP_{\lambda}$. 
\end{proof}

\subsection{Proof of Theorem~\ref{thm::chara_FAT_set}}

\begin{proof}
The proof relies on Theorem~\ref{thm::stand_meas} and Lemma~\ref{lem::blackwell_eqv}. For simplicity, we write the conditional $P_{\hat{\rY}|\rS,\rY}$ as its corresponding transition matrix $\bP$. Let $\bmu = (\Pr(\rS=1,\rY=1),\cdots,\Pr(\rS=A,\rY=C))$. The function \eqref{eq::disc_func} in our setting can be written as:
\begin{align}
    \bp_{\bmu}(\hat{y})
    &= \left(\frac{\mu_{1,1} P_{(1,1),\hat{y}}}{\sum_{s,y} \mu_{s,y} P_{(s,y),\hat{y}}}, \cdots, \frac{\mu_{A,C} P_{(A,C),\hat{y}}}{\sum_{s,y} \mu_{s,y} P_{(s,y),\hat{y}}}\right).\\
    \bp_{\bmu}(x)
    &= \left(\frac{\mu_{1,1}\dif P_{\rX|\rS=1,\rY=1}}{\sum_{s,y} \mu_{s,y}\dif P_{\rX|\rS=s,\rY=y}}(x),\cdots,\frac{\mu_{A,C}\dif P_{\rX|\rS=A,\rY=C}}{\sum_{s,y} \mu_{s,y} \dif P_{\rX|\rS=s,\rY=y}}(x)\right).
\end{align}
Note that $\bp_{\bmu}(x) = \bg(x)$ due to Bayes' rule. By Lemma~\ref{lem::blackwell_eqv}, we can rewrite $\mathcal{C}$ in Definition~\ref{defn::achi_trans_mat} as
\begin{align}
    \mathcal{C}
    = \left\{\bP \mid P_{\hat{\rX}|\rS,\rY} \text{ is more informative than }\bP\right\}.
\end{align}
By Lemma~\ref{lem::blackwell_eqv} and Theorem~\ref{thm::stand_meas}, the above set is further equivalent to all transition matrices $\bP\in \mathcal{T}(C|AC)$ satisfying 
\begin{align}
\label{eq::phi_cha_trans}
    \sum_{\hat{y}=1}^C \phi\left(\frac{\mu_{1,1}P_{(1,1),\hat{y}}}{\sum_{s,y} \mu_{s,y} P_{(s,y),\hat{y}}}, \cdots, \frac{\mu_{A,C}P_{(A,C),\hat{y}}}{\sum_{s,y} \mu_{s,y} P_{(s,y),\hat{y}}}\right) \sum_{s,y} \mu_{s,y} P_{(s,y),\hat{y}}
    \leq \EE{\phi(\bg(\rX))}
\end{align}
for any function $\phi: \Delta_{AC} \to \Reals$ which is the maximum of finitely many linear functions. Now we can write $\phi(\bp) = \max_{i\in[k]} \left\{\ba_i^T \bp\right\}$---we ignore the bias term because $\ba_i^T \bp + b_i = (\ba_i + b_i \ones)^T \bp$. Then the inequality in \eqref{eq::phi_cha_trans} can be simplified as
\begin{align}
    \sum_{\hat{y}=1}^C \max_{i\in[k]} \left\{\ba_i^T\bLambda_{\mu} \bp_{\hat{y}}\right\}
    \leq \EE{\max_{i\in [k]}\{\ba_i^T \bg(\rX)\}},
\end{align}
where $\bp_{\hat{y}}$ is the $\hat{y}$-th column of $\bP$ and $\bLambda_{\mu}=\diag(\mu_{1,1},\cdots,\mu_{A,C})$. Finally, we can always normalize the above inequality so that each $\ba_i \in [-1,1]^{AC}$.
\end{proof}

\subsection{Proof of Theorem~\ref{thm::convergence}}
\begin{proof}
We denote
\begin{align*}
    f(\bP)
    &\defined \sum_{s=1}^A\sum_{y=1}^C \mu_{s,y} P_{(s,y),y},\\
    g(\bP;\ba_1,\cdots,\ba_k)
    &\defined \sum_{\hat{y}=1}^C \max_{i\in[k]} \left\{\ba_i^T\bLambda_{\mu} \bp_{\hat{y}}\right\}
    - \EE{\max_{i\in [k]}\{\ba_i^T \bg(\rX)\}},\\
    \mathcal{F}
    &\defined \mathcal{C}_k \cap \left\{\bP\in \mathcal{T}(C|AC)\mid \mathsf{SP} \leq \alpha_{\scalebox{.5}{\textnormal SP}}, \mathsf{EO} \leq \alpha_{\scalebox{.5}{\textnormal EO}}, \mathsf{OAE} \leq \alpha_{\scalebox{.5}{\textnormal OAE}}\right\}.
\end{align*}
Let $\mathcal{F}^t$ be the constraint set of $\bP$ at the $t$-th iteration of our algorithm. Note that $\mathcal{F} \subseteq \mathcal{F}^t$ by definition. If the algorithm stops at the $t$-th iteration, then for any $\{\ba_i \mid \ba_i \in [-1,1]^{AC}, i\in[k]\}$, $\bP^t$ satisfies
\begin{align*}
    g(\bP^t;\ba_1,\cdots,\ba_k) \leq 0,
\end{align*}
which implies $\bP^t \in \mathcal{F}$. Consequently, 
\begin{align*}
    f(\bP^t) 
    = \max_{\bP\in \mathcal{F}^t} f(\bP)
    \geq \max_{\bP\in \mathcal{F}} f(\bP)
    \geq f(\bP^t).
\end{align*}
As a result, $f(\bP^t) = \max_{\bP\in \mathcal{F}} f(\bP)$ so $\bP^t$ is an optimal solution of $\mathsf{FairFront}_k(\alpha_{\scalebox{.5}{\textnormal SP}}, \alpha_{\scalebox{.5}{\textnormal EO}}, \alpha_{\scalebox{.5}{\textnormal OAE}})$.

If the algorithm never stops, consider any convergent sub-sequence of $\bP^t$ that converges to a limit point $\bP^*\in\mathcal{T}(C|AC)$. To simplify our notation, we assume $\bP^t \to \bP^*$ as $t\to \infty$. Since $\{\mathcal{F}^t\}_{t\geq 1}$ is non-increasing and they all contain $\mathcal{F}$, there exists a set $\mathcal{F}^*$ such that 
\begin{align*}
    \lim_{t\to \infty}~\mathcal{F}^t
    = \mathcal{F}^*, \quad \mathcal{F} \subseteq \mathcal{F}^*.
\end{align*}
Therefore, we have
\begin{align*}
    f(\bP^*)
    = \lim_{t\to\infty}f(\bP^t)
    = \lim_{t\to\infty} \max_{\bP \in \mathcal{F}^t}~f(\bP)
    = \max_{\bP \in \mathcal{F}^*}~f(\bP).
\end{align*}
Since $\mathcal{F} \subseteq \mathcal{F}^*$, we have 
\begin{align*}
    f(\bP^*)
    =\max_{\bP \in \mathcal{F}^*}~f(\bP)
    \geq \max_{\bP \in \mathcal{F}}~f(\bP).
\end{align*}
If $\bP^* \not\in \mathcal{F}$, then there exists a $(\bar{\ba}_1,\cdots,\bar{\ba}_k)$, such that $g(\bP^*;\bar{\ba}_1,\cdots,\bar{\ba}_k) > 0$. Let $(\ba_{1,t},\cdots,\ba_{k,t})$ be the output of Step 2 at $t$-th iteration. Since $\bP^* \in \mathcal{F}^t$ for all $t$, we have
\begin{align}
\label{eq::gleq0}
    g(\bP^*;\ba_{1,t},\cdots,\ba_{k,t}) \leq 0.
\end{align}
By the optimality of $(\ba_{1,t},\cdots,\ba_{k,t})$, we have
\begin{align}
\label{eq::ine_g_opt_a}
    g(\bP^t;\ba_{1,t},\cdots,\ba_{k,t})
    \geq g(\bP^t;\bar{\ba}_{1},\cdots,\bar{\ba}_{k}). 
\end{align}
Suppose that some sub-sequence of $(\ba_{1,t},\cdots,\ba_{k,t})$ converges to a vector $(\ba_{1}^*,\cdots,\ba_{k}^*)$. For the sake of simplicity, we assume $(\ba_{1,t},\cdots,\ba_{k,t}) \to (\ba_{1}^*,\cdots,\ba_{k}^*)$ as $t\to \infty$. On the one hand, taking limit of $t\to\infty$ on both sides of \eqref{eq::ine_g_opt_a} leads to
\begin{align*}
    g(\bP^*;\ba_{1}^*,\cdots,\ba_{k}^*)
    \geq g(\bP^*;\bar{\ba}_{1},\cdots,\bar{\ba}_{k}). 
\end{align*}
On the other hand, taking limit of $t\to\infty$ on both sides of \eqref{eq::gleq0} leads to
\begin{align*}
    g(\bP^*;\ba_{1}^*,\cdots,\ba_{k}^*) \leq 0.
\end{align*}
Therefore,
\begin{align*}
    0
    \geq g(\bP^*;\ba_{1}^*,\cdots,\ba_{k}^*)
    \geq g(\bP^*;\bar{\ba}_{1},\cdots,\bar{\ba}_{k})
    >0,
\end{align*}
which is impossible. Therefore, $\bP^* \in \mathcal{F}$ and, as a result, we have
\begin{align*}
    f(\bP^*)
    =\max_{\bP \in \mathcal{F}^*}~f(\bP)
    \geq \max_{\bP \in \mathcal{F}}~f(\bP)
    \geq f(\bP^*)
    \implies \max_{\bP \in \mathcal{F}}~f(\bP)
    = f(\bP^*).
\end{align*}
\end{proof}

\subsection{Additional Results}

We establish basic properties of $\mathcal{C}$ and $\mathsf{FairFront}(\alpha_{\scalebox{.5}{\textnormal SP}}, \alpha_{\scalebox{.5}{\textnormal EO}}, \alpha_{\scalebox{.5}{\textnormal OAE}})$ in the following lemma. 
\begin{lem}
\label{lem::prop_conv_C}
$\mathcal{C}$ is a convex subset of $\mathcal{T}(C|AC)$ and $\mathsf{FairFront}(\alpha_{\scalebox{.5}{\textnormal SP}}, \alpha_{\scalebox{.5}{\textnormal EO}}, \alpha_{\scalebox{.5}{\textnormal OAE}})$ is a concave function w.r.t. $\alpha_{\scalebox{.5}{\textnormal SP}}, \alpha_{\scalebox{.5}{\textnormal EO}}, \alpha_{\scalebox{.5}{\textnormal OAE}}$. Here the constants $A$ and $C$ denote the number of protected groups and the number of classes. 
\end{lem}

Next, we discuss a special case---$\rX$ is discrete---under which $\mathcal{C}$ has a simple characterization.
\begin{rem}
\label{rem:disc_X}
If $\rX$ is a \emph{discrete} variable with a \emph{finite support} $[D]$, we can write $P_{\rX|\rS,\rY}$ as a transition matrix $\boldsymbol\Phi \in \mathcal{T}(D|AC)$. By introducing an auxiliary variable $\bM \in \mathcal{T}(C|D)$, we can write $\bP \in \mathcal{C}$ equivalently as linear constraints: $\bP = \boldsymbol\Phi \bM$ by using the last condition of Lemma~\ref{lem::blackwell_eqv}. Consequently, Proposition~\ref{prop::FATO_linear} boils down to a linear program. However, this characterization fails to generalize to continuous data because $\boldsymbol\Phi$ and $\bM$ will have an infinite dimension; for categorical data, this characterization suffers from the curse of dimensionality since the support size of $\rX$ grows exponentially fast w.r.t. the number of features.
\end{rem}

\section{Details on the Experimental Results}
\label{sec::append_exp}

\subsection{Additional Experiments}

\begin{figure*}[t]
\centering
\includegraphics[width=0.40\linewidth]{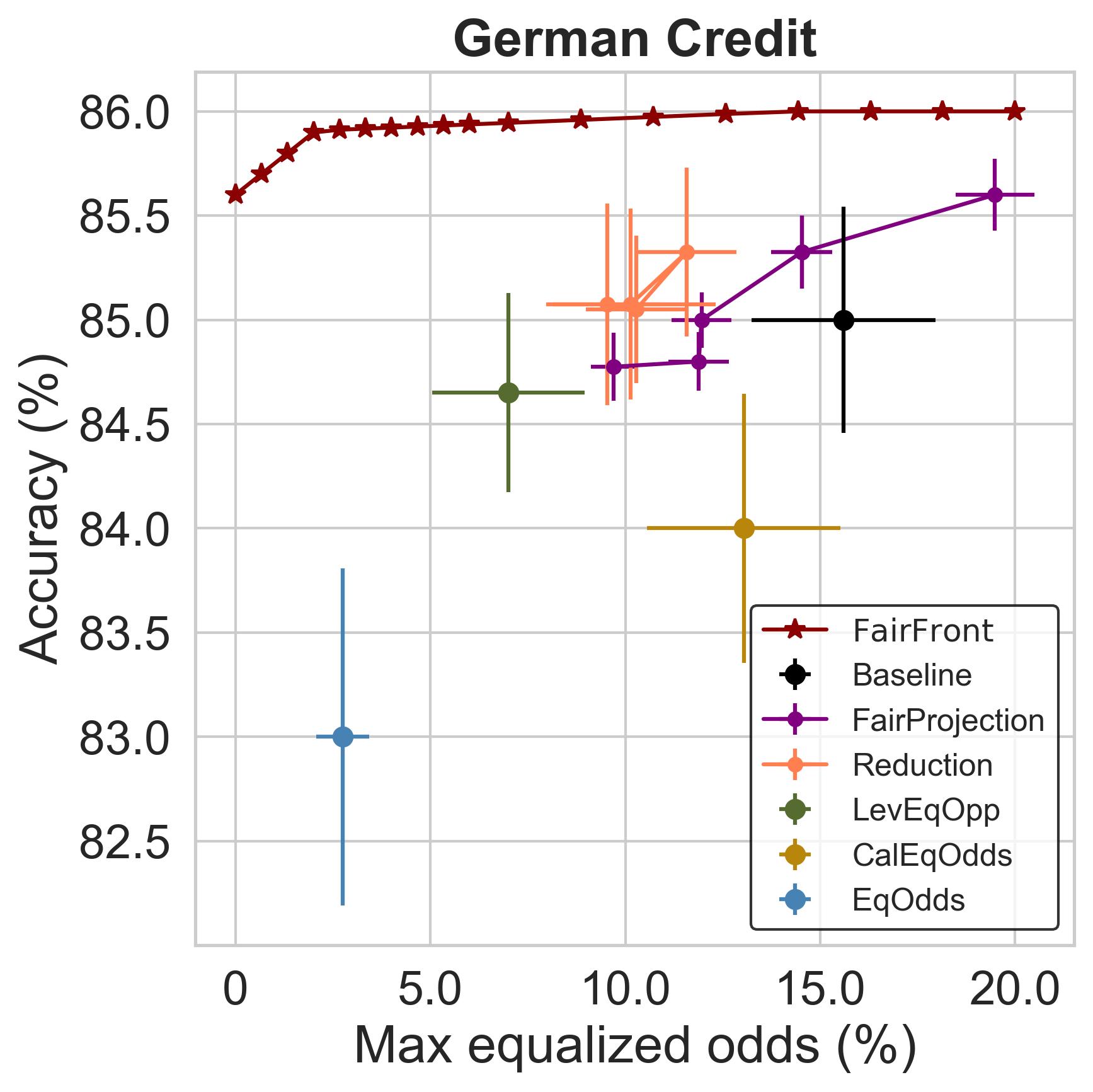}
\includegraphics[width=0.40\linewidth]{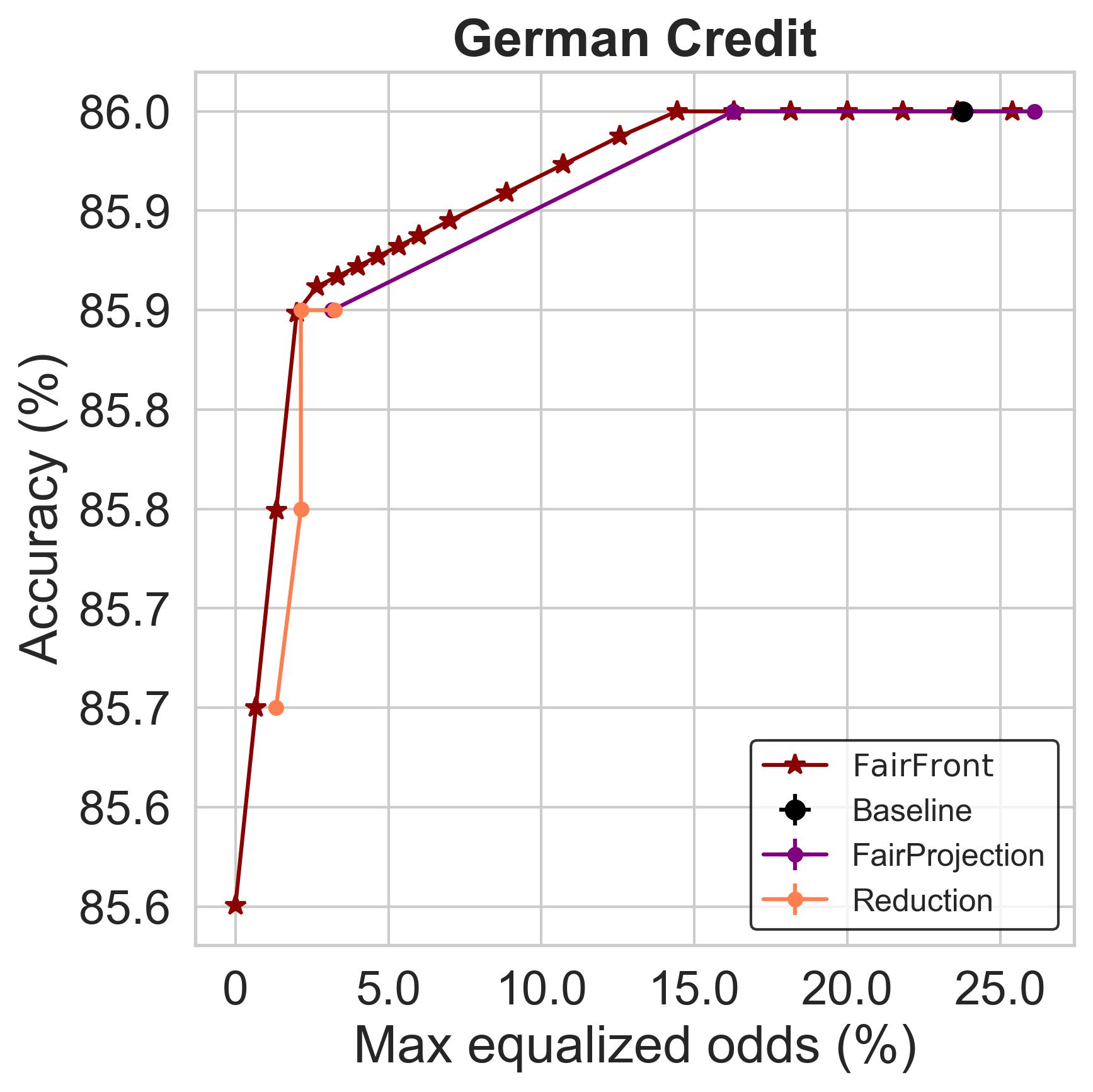}
\caption{We reproduce our experiments on the German Credit dataset. Our observation is consistent with those on the previous two datasets---the fairness-accuracy curves given by SOTA fairness interventions, such as \texttt{Reduction} and \texttt{FairProjection}, are close to the information-theoretic limit.
}
\label{Fig::German}
\end{figure*}

\begin{figure*}[t]
\centering
\includegraphics[width=0.40\linewidth]{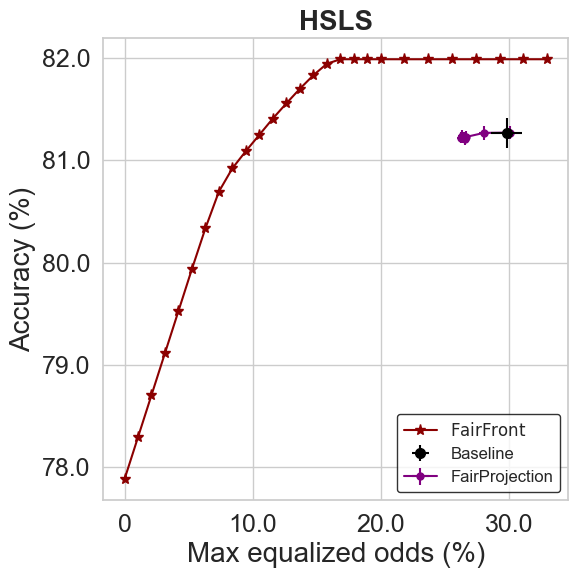}
\includegraphics[width=0.40\linewidth]{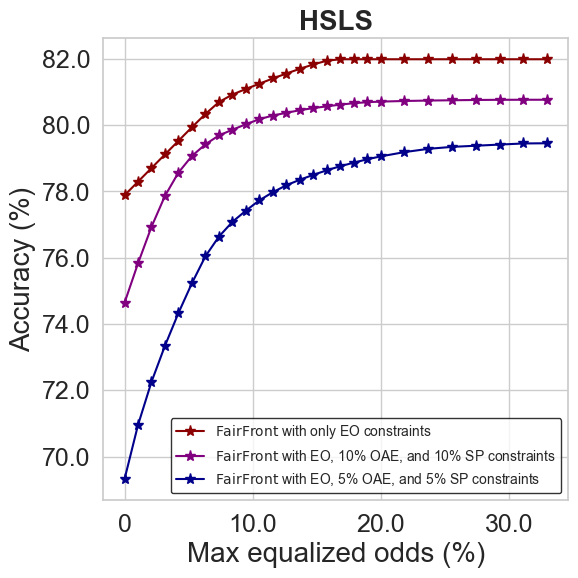}
\caption{We reproduce our experiments on the HSLS dataset with multi-group and multi-label pre-processing. On the right, we also demonstrate that \texttt{FairFront} can take into account multiple fairness considerations at once. We show how the fairness-accuracy curve changes as we add new types of group fairness constraints (i.e., adding OAE and SP constraints in addition to EO).
}
\label{Fig::HSLS}
\end{figure*}

In this section, we present additional experimental results to further support our findings. We reproduce our experimental results on the German Credit dataset \citep{bache2013uci} and HSLS (High School Longitudinal Study) dataset \citep{ingels2011high,jeong2022fairness} in Figure~\ref{Fig::German} and Figure~\ref{Fig::HSLS}. In particular, the HSLS dataset experiment is a multi-group, multi-label experiment. Our observation is consistent with those on the previous two datasets---the fairness-accuracy curves given by SOTA fairness interventions, such as \texttt{Reduction} and \texttt{FairProjection}, are close to the information-theoretic limit.

\subsection{Dataset}

\paragraph{Adult.} 
We use sex (female or male) as the group attribute and income ($>50\text{K}$ or $<=50\text{K}$) as the target for prediction. We use sex, hours-per-week, education-num, age, marital status, relationship status (husband or wife) as the input features---we include the group attribute as an input feature. We group age into a total of 12 disjoint intervals: $[0, 20)$, $[20,25), \cdots, [65,70), [70,\infty)$; we group hours-per-week into a total of 14 disjoint intervals: $[0,10), [10,15),\cdots,[65,70), [70,\infty)$.

\paragraph{COMPAS.}
We use race (African-American or Caucasian) as the group attribute and is\_recid (recid. or no recid.) as the target for prediction. We use race, age, c\_charge\_degree, sex, priors\_count, c\_jail\_in, c\_jail\_out as the input features---we include the group attribute as an input feature. We use the last two features by taking their difference to be their length\_of\_stay. We remove entries where COMPAS case could not be found (is\_recid = -1) and entries with inconsistent arrest information. We also binarize sex and remove traffic offenses. We quantize age the same way we do in the Adult dataset and quantize length\_of\_stay by every 30 days and let 0 be a separate category.

\paragraph{German Credit.}
We use age (below or above 25 years old) as the group attribute and the credit column, which represents whether the loan was a good decision, as the target for prediction. We use loan duration in month, credit amount, age, number of existing credits at this bank, sex, credit history, savings, and length of present employment as input features.
We include the group attribute age as an input feature. We group credit amount into three disjoint intervals: [0, 5000), [5000, 10000), [10000,$\infty$). We group duration of loan into two categories: under 36 months and over 36 months.

\paragraph{HSLS.}
We use race as the group attribute and mathematics test score (number of questions answered correctly out of 72) as the target for prediction. This is a multi-group and multi-label dataset. The entire population is grouped by 4 categories: White, Asian, African American, and Others. We seek to predict the mathematics test performance from a set of attributes, including the scale of student's mathematical identity, scale of student's mathematics utility, scale of one's mathematics self-efficacy, parent's education, parent's income, scale of student's sense of school belonging, race, and sex. Note that we include the group attribute as an input feature. We group the target column (estimated number of questions answered correctly) into a total of 5 disjoint intervals: $[0, 30)$, $[30,40), [40,50),[50,60),[60,\infty)$; we group the scale of student's mathematical identity, mathematics utility, mathematics self-efficacy, and sense of school belonging into a total of 4 disjoint intervals, characterized by standard deviations away from the mean: $(-\infty,-1), [-1,0),[0,1),[1,\infty)$.

\subsection{Benchmark}

Each benchmark method's hyper-parameter values are provided below. Each point in Figure~\ref{Fig::Fair_Frontier} for \texttt{Baseline}, \texttt{EqOdds}, \texttt{CalEqOdds}, \texttt{Reduction}, \texttt{LevEqOpp}, and \texttt{FairProjection} is obtained by applying the obtained classifier to 10 different test sets. For the Adult dataset, we use Random Forest with n\_estimators=15, min\_samples\_leaf=3, criterion = log\_loss, bootstrap = False as our baseline classifier; for the COMPAS dataset, we use Random Forest with n\_estimators = 17 as our baseline classifier. For the German Credit dataset, we use Random Forest with n\_estimators=100,min\_samples\_split =2,min\_samples\_leaf=1 as our baseline classifier. They are all implemented by Scikit-learn \citep{scikitlearn}.


\paragraph{EqOdds \citep{hardt2016equality}.} We use AIF360 implementation of \texttt{EqOddsPostprocessing} and the default hyper-parameter setup.

\paragraph{CalEqOdds \citep{pleiss2017fairness}.} We use AIF360 implementation of \texttt{CalibratedEqOddsPostprocessing} and the default hyper-parameter setup.

\paragraph{Reduction \citep{agarwal2018reductions}.} We use AIF360 implementation of \texttt{ExponentiatedGradientReduction}. We vary the allowed fairness constraint violation $\epsilon \in \{ 0.001, 0.01, 0.2, 0.5, 1, 2, 5, 10, 15\}$ for Adult dataset and $\epsilon \in \{0.001, 0.01, 0.2, 0.5, 1, 2, 5, 10, 15\}$ for Adult with missing values. We vary $\epsilon \in \{ 0.001, 2, 5, 10, 15, 20, 25, 30, 35, 40\}$ for COMPAS to obtain a fairness-accuracy curve, and $\epsilon \in \{0.001, 0.1, 0.5, 1, 2, 7, 8, 10, 15, 20, 25, 30\}$ for COMPAS with 50\% missing values in the minority group. We use $\epsilon\in \{20, 50, 80, 95\}$ for German Credit dataset and $\epsilon\in \{5,8,10,20,23\}$ when using Bayes Optimal classifier.

\paragraph{LevEqOpp \citep{chzhen2019leveraging}.} We use the Python implementation of \texttt{LevEqopp} from the Github repo in \citet{alghamdi2022beyond}. We follow the same hyperparameters setup as in the original method.

\paragraph{FairProjection \cite{alghamdi2022beyond}.} We use the implementation from the Github repo in \citet{alghamdi2022beyond} and set use\_protected = True. We use Random Forest with n\_estimators = 17 as the baseline classifier to predict $\rS$ from $(\rX,\rY)$. We set the list of fairness violation tolerance to be $\{0.07, 0.075, 0.08, 0.085, 0.09, 0.095, 0.1, 0.5, 0.75, 1.0 \}$ for Adult dataset and $\{0.02, 0.03, 0.04, 0.05, 0.06, 0.07, 0.08, 0.1, 0.5, 1.0\}$ for COMPAS dataset to obtain a fairness-accuracy curve. We set the list of fairness violation tolerance to be $\{0.005, 0.01, 0.02, 0.07, 0.1, 0.15\}$ on the German Credit dataset experiment, and $\{0.0001, 0.001, 0.005, 0.01, 0.015, 0.02, 0.05\}$ when using a Bayes optimal baseline classifier. For the HSLS dataset, the list of tolerance is $\{0.06, 0.065, 0.07, 0.08, 0.1, 0.09, 0.15, 0.2, 0.3\}$.

\input{relatedwork.tex}

%% file: relatedwork.tex
\section{More on Related Work}

We provide a detailed comparison with existing work on fairness Pareto frontier and ML uncertainty in this section.

\subsection{Fairness Pareto Frontier}

We present in Table~\ref{tabel:comparison} a detailed comparison of our approach with previous studies that have investigated the fairness Pareto frontier and fair Bayes optimal classifier. In short, our approach is different from this line of research as it simultaneously combines several important aspects: it is applicable to \emph{multiclass} classification problems with \emph{multiple} protected groups; it \emph{avoids disparate treatment} by not requiring the classifier to use group attributes as an input variable; and it can handle \emph{multiple} fairness constraints simultaneously and produce fairness-accuracy trade-off curves (instead of a single point).

\begin{table*}[t]
\small
\centering
\resizebox{0.77\textwidth}{!}{
\renewcommand{\arraystretch}{1.25}
\begin{tabular}{lccccc}

\toprule

&

Multiclass

& 

Multigroup

&

\sccell{c}{
Avoid\\
disparate treatment
}

&

Multi-constraint

&

Curve

\\
\toprule

\citet{hardt2016equality}

&

\XSolidBold

&

\CheckmarkBold

&

\XSolidBold

&

\XSolidBold

&

\XSolidBold

\\ \midrule 

\citet{corbett2017algorithmic}

&

\XSolidBold

&

\CheckmarkBold

&

\XSolidBold

&

\XSolidBold

&

\XSolidBold

\\ \midrule

\citet{menon2018cost}

&

\XSolidBold

&

\XSolidBold

&

\CheckmarkBold

&

\XSolidBold

&

\CheckmarkBold

\\ \midrule 

\citet{chzhen2019leveraging}

&

\XSolidBold

&

\XSolidBold

&
\CheckmarkBold

&

\XSolidBold

&

\XSolidBold

\\ \midrule

\citet{yang2020fairness}

&

\CheckmarkBold

&

\CheckmarkBold

&

\XSolidBold

&

\CheckmarkBold

&

\CheckmarkBold

\\ \midrule

\citet{zeng2022bayes}

&

\XSolidBold

&
\CheckmarkBold

&
\XSolidBold

&
\XSolidBold

&

\CheckmarkBold

\\ \midrule

\citet{zeng2022fair}

&

\XSolidBold

&

\CheckmarkBold

&

\XSolidBold

&

\XSolidBold

&

\XSolidBold

\\ \midrule

Our approach

&

\CheckmarkBold

&

\CheckmarkBold

&

\CheckmarkBold

&

\CheckmarkBold

&

\CheckmarkBold

\\

\bottomrule 
\end{tabular}
}
\caption{Comparison with existing work that investigate the fairness Pareto frontier. \textbf{Multiclass/multigroup}: can handle multiclass classification problems with multiple protected groups; \textbf{Avoid disparate treatment}: not require the classifier to use group attributes as an input variable; \textbf{Multi-constraint}: can handle multiple (group) fairness constraints simultaneously; \textbf{Curve}: produce fairness-accuracy trade-off curves (instead of a single point).}
\label{tabel:comparison}
\end{table*}

\subsection{Aleatoric and Epistemic Uncertainty}
\label{subsec::related_uncertainty}

In this paper, we divide algorithmic discrimination into aleatoric and epistemic discrimination. We borrow this notion from ML uncertainty literature \citep[see][for a survey]{hullermeier2021aleatoric}. Here we provide a detailed comparison between them.

In terms of their definitions, epistemic uncertainty arises from a lack of knowledge about the best model, such as the Bayes predictor, while epistemic discrimination results from a lack of knowledge about the optimal ``fair'' predictive model. On the other hand, aleatoric uncertainty is the irreducible part of uncertainty caused by the random relationship between input features and label, while aleatoric discrimination is due to inherent biases in the data-generating distribution.

In terms of their characterization, epistemic uncertainty can in principle be reduced by including additional information (e.g., more data); epistemic discrimination can be reduced in a similar approach since a data scientist can choose a more effective fairness-intervention algorithm with access to more information.

Finally, in the infinite sample regime, a consistent learner will be able to remove all epistemic uncertainty, assuming the model class is large enough and there are no computational constraints. Analogously, we demonstrate in Figure~\ref{Fig::BayesOptimal} that when the underlying distribution is known, SOTA fairness interventions are able to eliminate epistemic discrimination as their fairness-accuracy curves are close to the fair front.

\section{More on Future Work}
\label{append:future}

In this paper, we present an upper bound estimate for $\mathsf{FairFront}$ in Algorithm~\ref{alg:FATO}. It is important to note that this estimate may be subjected to errors originating from various sources. These include (i) the approximation error of the function $g$, (ii) estimation errors from computing the expectation in \eqref{eq::chara_C} with a finite dataset, and (iii) the influence of hyperparameters, $T$ (number of running iterations of Algorithm~\ref{alg:FATO}) and $k$ (number of segments in the piece-wise linear functions). Regarding the dependence on $T$, our Theorem~\ref{thm::convergence} ensures the algorithm's asymptotic convergence as $T \to \infty$. However, we have not established a proof for its behavior at a finite $T$. Regarding the dependence on $k$, we conjecture that $k = A * C$ should suffice, where $A$ is the number of protected groups and $C$ is the number of labels. While Blackwell proved this result for $k=2$ in Theorem 10 of \citet{blackwell1953equivalent}, an extension of this proof to a general value of $k$ appears to remain an open problem.

We define aleatoric and epistemic discrimination with respect to the entire population. Investigating their per-instance counterparts and the relationship to individual fairness would be a compelling area of future research. Additionally, a more nuanced analysis of aleatoric and epistemic discrimination is desirable, further breaking them down into fine-grained components. For instance, epistemic discrimination may be attributed to various factors including limited training data, noisy observations of labels or sensitive attributes, and limitations of learning algorithms. Finally, investigating other criteria, such as scalability, generalization, and robustness in evaluating existing fairness interventions is a significant topic for future exploration.

%% file: main.bbl
\begin{thebibliography}{}

\bibitem[Agarwal et~al., 2018]{agarwal2018reductions}
Agarwal, A., Beygelzimer, A., Dud{\'\i}k, M., Langford, J., and Wallach, H.
  (2018).
\newblock A reductions approach to fair classification.
\newblock In {\em International Conference on Machine Learning}, pages 60--69.
  PMLR.

\bibitem[Alghamdi et~al., 2020]{alghamdi2020model}
Alghamdi, W., Asoodeh, S., Wang, H., Calmon, F.~P., Wei, D., and Ramamurthy,
  K.~N. (2020).
\newblock Model projection: Theory and applications to fair machine learning.
\newblock In {\em 2020 IEEE International Symposium on Information Theory
  (ISIT)}, pages 2711--2716. IEEE.

\bibitem[Alghamdi et~al., 2022]{alghamdi2022beyond}
Alghamdi, W., Hsu, H., Jeong, H., Wang, H., Michalak, P.~W., Asoodeh, S., and
  Calmon, F.~P. (2022).
\newblock Beyond {Adult} and {COMPAS}: Fair multi-class prediction via
  information projection.
\newblock In {\em Advances in Neural Information Processing Systems}.

\bibitem[Angwin et~al., 2016]{angwin2016machine}
Angwin, J., Larson, J., Mattu, S., and Kirchner, L. (2016).
\newblock Machine bias.
\newblock {\em ProPublica}.

\bibitem[Bache and Lichman, 2013]{bache2013uci}
Bache, K. and Lichman, M. (2013).
\newblock {UCI Machine Learning Repository}.

\bibitem[Bellamy et~al., 2018]{bellamy2018ai}
Bellamy, R.~K., Dey, K., Hind, M., Hoffman, S.~C., Houde, S., Kannan, K.,
  Lohia, P., Martino, J., Mehta, S., Mojsilovic, A., et~al. (2018).
\newblock Ai fairness 360: An extensible toolkit for detecting, understanding,
  and mitigating unwanted algorithmic bias.
\newblock {\em arXiv preprint arXiv:1810.01943}.

\bibitem[Bellamy et~al., 2019]{bellamy2019ai}
Bellamy, R.~K., Dey, K., Hind, M., Hoffman, S.~C., Houde, S., Kannan, K.,
  Lohia, P., Martino, J., Mehta, S., Mojsilovi{\'c}, A., et~al. (2019).
\newblock Ai fairness 360: An extensible toolkit for detecting and mitigating
  algorithmic bias.
\newblock {\em IBM Journal of Research and Development}, 63(4/5):4--1.

\bibitem[Berk et~al., 2021]{berk2021fairness}
Berk, R., Heidari, H., Jabbari, S., Kearns, M., and Roth, A. (2021).
\newblock Fairness in criminal justice risk assessments: The state of the art.
\newblock {\em Sociological Methods \& Research}, 50(1):3--44.

\bibitem[Blackwell, 1951]{blackwell1951comparison}
Blackwell, D. (1951).
\newblock Comparison of experiments.
\newblock {\em Proceedings of the Second Berkeley Symposium on Mathematical
  Statistics and Probability}, pages 93--102.

\bibitem[Blackwell, 1953]{blackwell1953equivalent}
Blackwell, D. (1953).
\newblock Equivalent comparisons of experiments.
\newblock {\em The annals of mathematical statistics}, pages 265--272.

\bibitem[Blodgett et~al., 2020]{blodgett2020language}
Blodgett, S.~L., Barocas, S., Daum{\'e}~III, H., and Wallach, H. (2020).
\newblock Language (technology) is power: A critical survey of" bias" in nlp.
\newblock {\em arXiv preprint arXiv:2005.14050}.

\bibitem[Blum and Stangl, 2019]{blum2019recovering}
Blum, A. and Stangl, K. (2019).
\newblock Recovering from biased data: Can fairness constraints improve
  accuracy?
\newblock {\em arXiv preprint arXiv:1912.01094}.

\bibitem[Calmon et~al., 2017]{calmon2017optimized}
Calmon, F., Wei, D., Vinzamuri, B., Natesan~Ramamurthy, K., and Varshney, K.~R.
  (2017).
\newblock Optimized pre-processing for discrimination prevention.
\newblock {\em Advances in neural information processing systems}, 30.

\bibitem[Cam, 1964]{le1964sufficiency}
Cam, L.~L. (1964).
\newblock Sufficiency and approximate sufficiency.
\newblock {\em The Annals of Mathematical Statistics}, pages 1419--1455.

\bibitem[Caton et~al., 2022]{caton2022impact}
Caton, S., Malisetty, S., and Haas, C. (2022).
\newblock Impact of imputation strategies on fairness in machine learning.
\newblock {\em Journal of Artificial Intelligence Research}, 74:1011--1035.

\bibitem[Celis et~al., 2019]{celis2019classification}
Celis, L.~E., Huang, L., Keswani, V., and Vishnoi, N.~K. (2019).
\newblock Classification with fairness constraints: A meta-algorithm with
  provable guarantees.
\newblock In {\em Proceedings of the conference on fairness, accountability,
  and transparency}, pages 319--328.

\bibitem[Chen et~al., 2018]{chen2018my}
Chen, I., Johansson, F.~D., and Sontag, D. (2018).
\newblock Why is my classifier discriminatory?
\newblock {\em Advances in neural information processing systems}, 31.

\bibitem[Chouldechova, 2017]{chouldechova2017fair}
Chouldechova, A. (2017).
\newblock Fair prediction with disparate impact: A study of bias in recidivism
  prediction instruments.
\newblock {\em Big data}, 5(2):153--163.

\bibitem[Chzhen et~al., 2019]{chzhen2019leveraging}
Chzhen, E., Denis, C., Hebiri, M., Oneto, L., and Pontil, M. (2019).
\newblock Leveraging labeled and unlabeled data for consistent fair binary
  classification.
\newblock {\em Advances in Neural Information Processing Systems}, 32.

\bibitem[Cohen et~al., 1998]{cohen1998comparisons}
Cohen, J., Kempermann, J.~H., and Zbaganu, G. (1998).
\newblock {\em Comparisons of stochastic matrices with applications in
  information theory, statistics, economics and population}.
\newblock Springer Science \& Business Media.

\bibitem[Corbett-Davies et~al., 2017]{corbett2017algorithmic}
Corbett-Davies, S., Pierson, E., Feller, A., Goel, S., and Huq, A. (2017).
\newblock Algorithmic decision making and the cost of fairness.
\newblock In {\em Proceedings of the 23rd acm sigkdd international conference
  on knowledge discovery and data mining}, pages 797--806.

\bibitem[Ding et~al., 2021]{ding2021retiring}
Ding, F., Hardt, M., Miller, J., and Schmidt, L. (2021).
\newblock Retiring adult: New datasets for fair machine learning.
\newblock {\em Advances in neural information processing systems},
  34:6478--6490.

\bibitem[Dutta et~al., 2020]{dutta2020there}
Dutta, S., Wei, D., Yueksel, H., Chen, P.-Y., Liu, S., and Varshney, K. (2020).
\newblock Is there a trade-off between fairness and accuracy? a perspective
  using mismatched hypothesis testing.
\newblock In {\em International Conference on Machine Learning}, pages
  2803--2813. PMLR.

\bibitem[Dwork et~al., 2018]{dwork2018decoupled}
Dwork, C., Immorlica, N., Kalai, A.~T., and Leiserson, M. (2018).
\newblock Decoupled classifiers for group-fair and efficient machine learning.
\newblock In {\em Conference on fairness, accountability and transparency},
  pages 119--133. PMLR.

\bibitem[Feldman et~al., 2015]{feldman2015certifying}
Feldman, M., Friedler, S.~A., Moeller, J., Scheidegger, C., and
  Venkatasubramanian, S. (2015).
\newblock Certifying and removing disparate impact.
\newblock In {\em proceedings of the 21th ACM SIGKDD international conference
  on knowledge discovery and data mining}, pages 259--268.

\bibitem[Fernando et~al., 2021]{fernando2021missing}
Fernando, M.-P., C{\`e}sar, F., David, N., and Jos{\'e}, H.-O. (2021).
\newblock Missing the missing values: The ugly duckling of fairness in machine
  learning.
\newblock {\em International Journal of Intelligent Systems}, 36(7):3217--3258.

\bibitem[Fogliato et~al., 2020]{fogliato2020fairness}
Fogliato, R., Chouldechova, A., and G’Sell, M. (2020).
\newblock Fairness evaluation in presence of biased noisy labels.
\newblock In {\em International Conference on Artificial Intelligence and
  Statistics}, pages 2325--2336. PMLR.

\bibitem[Friedler et~al., 2019]{friedler2019comparative}
Friedler, S.~A., Scheidegger, C., Venkatasubramanian, S., Choudhary, S.,
  Hamilton, E.~P., and Roth, D. (2019).
\newblock A comparative study of fairness-enhancing interventions in machine
  learning.
\newblock In {\em Proceedings of the conference on fairness, accountability,
  and transparency}, pages 329--338.

\bibitem[Globus-Harris et~al., 2023]{globus2023multicalibrated}
Globus-Harris, I., Gupta, V., Jung, C., Kearns, M., Morgenstern, J., and Roth,
  A. (2023).
\newblock Multicalibrated regression for downstream fairness.
\newblock In {\em Proceedings of the 2023 AAAI/ACM Conference on AI, Ethics,
  and Society}, pages 259--286.

\bibitem[Gopalan et~al., 2021]{gopalan2021omnipredictors}
Gopalan, P., Kalai, A.~T., Reingold, O., Sharan, V., and Wieder, U. (2021).
\newblock Omnipredictors.
\newblock {\em arXiv preprint arXiv:2109.05389}.

\bibitem[Hardt et~al., 2016]{hardt2016equality}
Hardt, M., Price, E., and Srebro, N. (2016).
\newblock Equality of opportunity in supervised learning.
\newblock In {\em Advances in Neural Information Processing Systems},
  volume~29.

\bibitem[Horst and Thoai, 1999]{horst1999dc}
Horst, R. and Thoai, N.~V. (1999).
\newblock Dc programming: overview.
\newblock {\em Journal of Optimization Theory and Applications}, 103(1):1--43.

\bibitem[Hort et~al., 2022]{hort2022bia}
Hort, M., Chen, Z., Zhang, J.~M., Sarro, F., and Harman, M. (2022).
\newblock Bia mitigation for machine learning classifiers: A comprehensive
  survey.
\newblock {\em arXiv preprint arXiv:2207.07068}.

\bibitem[Hu et~al., 2023]{hu2023omnipredictors}
Hu, L., Navon, I. R.~L., Reingold, O., and Yang, C. (2023).
\newblock Omnipredictors for constrained optimization.
\newblock In {\em International Conference on Machine Learning}, pages
  13497--13527. PMLR.

\bibitem[H{\"u}llermeier and Waegeman, 2021]{hullermeier2021aleatoric}
H{\"u}llermeier, E. and Waegeman, W. (2021).
\newblock Aleatoric and epistemic uncertainty in machine learning: An
  introduction to concepts and methods.
\newblock {\em Machine Learning}, 110(3):457--506.

\bibitem[Ingels et~al., 2011]{ingels2011high}
Ingels, S.~J., Pratt, D.~J., Herget, D.~R., Burns, L.~J., Dever, J.~A., Ottem,
  R., Rogers, J.~E., Jin, Y., and Leinwand, S. (2011).
\newblock High school longitudinal study of 2009 (hsls: 09): Base-year data
  file documentation. nces 2011-328.
\newblock {\em National Center for Education Statistics}.

\bibitem[Jacobs and Wallach, 2021]{jacobs2021measurement}
Jacobs, A.~Z. and Wallach, H. (2021).
\newblock Measurement and fairness.
\newblock In {\em Proceedings of the 2021 ACM conference on fairness,
  accountability, and transparency}, pages 375--385.

\bibitem[Jeong et~al., 2022]{jeong2022fairness}
Jeong, H., Wang, H., and Calmon, F.~P. (2022).
\newblock Fairness without imputation: A decision tree approach for fair
  prediction with missing values.
\newblock In {\em Proceedings of the AAAI Conference on Artificial
  Intelligence}, volume~36, pages 9558--9566.

\bibitem[Jiang and Nachum, 2020]{jiang2020identifying}
Jiang, H. and Nachum, O. (2020).
\newblock Identifying and correcting label bias in machine learning.
\newblock In {\em International Conference on Artificial Intelligence and
  Statistics}, pages 702--712. PMLR.

\bibitem[Jiang et~al., 2020]{jiang2020wasserstein}
Jiang, R., Pacchiano, A., Stepleton, T., Jiang, H., and Chiappa, S. (2020).
\newblock Wasserstein fair classification.
\newblock In {\em Uncertainty in Artificial Intelligence}, pages 862--872.
  PMLR.

\bibitem[Kallus et~al., 2022]{kallus2022assessing}
Kallus, N., Mao, X., and Zhou, A. (2022).
\newblock Assessing algorithmic fairness with unobserved protected class using
  data combination.
\newblock {\em Management Science}, 68(3):1959--1981.

\bibitem[Kamiran and Calders, 2012]{kamiran2012data}
Kamiran, F. and Calders, T. (2012).
\newblock Data preprocessing techniques for classification without
  discrimination.
\newblock {\em Knowledge and information systems}, 33(1):1--33.

\bibitem[Katzman et~al., 2023]{katzman2023taxonomizing}
Katzman, J., Wang, A., Scheuerman, M., Blodgett, S.~L., Laird, K., Wallach, H.,
  and Barocas, S. (2023).
\newblock Taxonomizing and measuring representational harms: A look at image
  tagging.
\newblock {\em arXiv preprint arXiv:2305.01776}.

\bibitem[Kearns et~al., 2018]{kearns2018preventing}
Kearns, M., Neel, S., Roth, A., and Wu, Z.~S. (2018).
\newblock Preventing fairness gerrymandering: Auditing and learning for
  subgroup fairness.
\newblock In {\em International Conference on Machine Learning}, pages
  2564--2572. PMLR.

\bibitem[Kim et~al., 2020]{kim2020fact}
Kim, J.~S., Chen, J., and Talwalkar, A. (2020).
\newblock Fact: A diagnostic for group fairness trade-offs.
\newblock In {\em International Conference on Machine Learning}, pages
  5264--5274. PMLR.

\bibitem[Kim et~al., 2019]{kim2019multiaccuracy}
Kim, M.~P., Ghorbani, A., and Zou, J. (2019).
\newblock Multiaccuracy: Black-box post-processing for fairness in
  classification.
\newblock In {\em Proceedings of the 2019 AAAI/ACM Conference on AI, Ethics,
  and Society}, pages 247--254.

\bibitem[Kleinberg et~al., 2016]{kleinberg2016inherent}
Kleinberg, J., Mullainathan, S., and Raghavan, M. (2016).
\newblock Inherent trade-offs in the fair determination of risk scores.
\newblock {\em arXiv preprint arXiv:1609.05807}.

\bibitem[Lowy et~al., 2021]{lowy2021fermi}
Lowy, A., Pavan, R., Baharlouei, S., Razaviyayn, M., and Beirami, A. (2021).
\newblock Fermi: Fair empirical risk minimization via exponential {R{\'e}nyi}
  mutual information.
\newblock {\em arXiv preprint arXiv:2102.12586}.

\bibitem[Martinez et~al., 2020]{martinez2020minimax}
Martinez, N., Bertran, M., and Sapiro, G. (2020).
\newblock Minimax pareto fairness: A multi objective perspective.
\newblock In {\em International Conference on Machine Learning}, pages
  6755--6764. PMLR.

\bibitem[Mayson, 2019]{mayson2019bias}
Mayson, S.~G. (2019).
\newblock Bias in, bias out.
\newblock {\em The Yale Law Journal}, 128(8):2218--2300.

\bibitem[Mehrotra and Celis, 2021]{mehrotra2021mitigating}
Mehrotra, A. and Celis, L.~E. (2021).
\newblock Mitigating bias in set selection with noisy protected attributes.
\newblock In {\em Proceedings of the 2021 ACM Conference on Fairness,
  Accountability, and Transparency}, pages 237--248.

\bibitem[Menon and Williamson, 2018]{menon2018cost}
Menon, A.~K. and Williamson, R.~C. (2018).
\newblock The cost of fairness in binary classification.
\newblock In {\em Conference on Fairness, Accountability and Transparency},
  pages 107--118. PMLR.

\bibitem[Pedregosa et~al., 2011]{scikitlearn}
Pedregosa, F., Varoquaux, G., Gramfort, A., Michel, V., Thirion, B., Grisel,
  O., Blondel, M., Prettenhofer, P., Weiss, R., Dubourg, V., Vanderplas, J.,
  Passos, A., Cournapeau, D., Brucher, M., Perrot, M., and Duchesnay, E.
  (2011).
\newblock Scikit-learn: Machine learning in {P}ython.
\newblock {\em Journal of Machine Learning Research}, 12:2825--2830.

\bibitem[Pleiss et~al., 2017]{pleiss2017fairness}
Pleiss, G., Raghavan, M., Wu, F., Kleinberg, J., and Weinberger, K.~Q. (2017).
\newblock On fairness and calibration.
\newblock {\em Advances in neural information processing systems}, 30.

\bibitem[Raginsky, 2011]{raginsky2011shannon}
Raginsky, M. (2011).
\newblock Shannon meets blackwell and le cam: Channels, codes, and statistical
  experiments.
\newblock In {\em 2011 IEEE International Symposium on Information Theory
  Proceedings}, pages 1220--1224. IEEE.

\bibitem[Rauh et~al., 2017]{rauh2017coarse}
Rauh, J., Banerjee, P.~K., Olbrich, E., Jost, J., Bertschinger, N., and
  Wolpert, D. (2017).
\newblock Coarse-graining and the blackwell order.
\newblock {\em Entropy}, 19(10):527.

\bibitem[Schelter et~al., 2019]{schelter2019fairprep}
Schelter, S., He, Y., Khilnani, J., and Stoyanovich, J. (2019).
\newblock Fairprep: Promoting data to a first-class citizen in studies on
  fairness-enhancing interventions.
\newblock {\em arXiv preprint arXiv:1911.12587}.

\bibitem[Shannon, 1958]{shannon1958note}
Shannon, C.~E. (1958).
\newblock A note on a partial ordering for communication channels.
\newblock {\em Information and control}, 1(4):390--397.

\bibitem[Shen et~al., 2016]{shen2016disciplined}
Shen, X., Diamond, S., Gu, Y., and Boyd, S. (2016).
\newblock Disciplined convex-concave programming.
\newblock In {\em 2016 IEEE 55th Conference on Decision and Control (CDC)},
  pages 1009--1014. IEEE.

\bibitem[Subramonian et~al., 2022]{subramoniandiscrimination}
Subramonian, A., Chang, K.-W., and Sun, Y. (2022).
\newblock On the discrimination risk of mean aggregation feature imputation in
  graphs.
\newblock {\em Advances in Neural Information Processing Systems}.

\bibitem[Suresh and Guttag, 2019]{suresh2019framework}
Suresh, H. and Guttag, J.~V. (2019).
\newblock A framework for understanding unintended consequences of machine
  learning.
\newblock {\em arXiv preprint arXiv:1901.10002}, 2:8.

\bibitem[Tomasev et~al., 2021]{tomasev2021fairness}
Tomasev, N., McKee, K.~R., Kay, J., and Mohamed, S. (2021).
\newblock Fairness for unobserved characteristics: Insights from technological
  impacts on queer communities.
\newblock In {\em Proceedings of the 2021 AAAI/ACM Conference on AI, Ethics,
  and Society}, pages 254--265.

\bibitem[Torgersen, 1991]{torgersen1991comparison}
Torgersen, E. (1991).
\newblock {\em Comparison of statistical experiments}, volume~36.
\newblock Cambridge University Press.

\bibitem[Ustun et~al., 2019]{ustun2019fairness}
Ustun, B., Liu, Y., and Parkes, D. (2019).
\newblock Fairness without harm: Decoupled classifiers with preference
  guarantees.
\newblock In {\em International Conference on Machine Learning}, pages
  6373--6382. PMLR.

\bibitem[Varshney, 2021]{varshney2021trustworthy}
Varshney, K.~R. (2021).
\newblock Trustworthy machine learning.
\newblock {\em Chappaqua, NY}.

\bibitem[Verma and Rubin, 2018]{verma2018fairness}
Verma, S. and Rubin, J. (2018).
\newblock Fairness definitions explained.
\newblock In {\em 2018 ieee/acm international workshop on software fairness
  (fairware)}, pages 1--7. IEEE.

\bibitem[Wang et~al., 2021]{wang2021split}
Wang, H., Hsu, H., Diaz, M., and Calmon, F.~P. (2021).
\newblock To split or not to split: The impact of disparate treatment in
  classification.
\newblock {\em IEEE Transactions on Information Theory}, 67(10):6733--6757.

\bibitem[Wang et~al., 2020]{wang2020robust}
Wang, S., Guo, W., Narasimhan, H., Cotter, A., Gupta, M., and Jordan, M.
  (2020).
\newblock Robust optimization for fairness with noisy protected groups.
\newblock {\em Advances in Neural Information Processing Systems},
  33:5190--5203.

\bibitem[Wang and Singh, 2021]{wang2021analyzing}
Wang, Y. and Singh, L. (2021).
\newblock Analyzing the impact of missing values and selection bias on
  fairness.
\newblock {\em International Journal of Data Science and Analytics},
  12(2):101--119.

\bibitem[Wei et~al., 2021]{wei2021optimized}
Wei, D., Ramamurthy, K.~N., and Calmon, F.~P. (2021).
\newblock Optimized score transformation for consistent fair classification.
\newblock {\em J. Mach. Learn. Res.}, 22:258--1.

\bibitem[Wick et~al., 2019]{wick2019unlocking}
Wick, M., Tristan, J.-B., et~al. (2019).
\newblock Unlocking fairness: a trade-off revisited.
\newblock {\em Advances in neural information processing systems}, 32.

\bibitem[Yang et~al., 2020]{yang2020fairness}
Yang, F., Cisse, M., and Koyejo, S. (2020).
\newblock Fairness with overlapping groups; a probabilistic perspective.
\newblock In {\em Advances in Neural Information Processing Systems},
  volume~33, pages 4067--4078.

\bibitem[Zafar et~al., 2019]{zafar2019fairness}
Zafar, M.~B., Valera, I., Gomez-Rodriguez, M., and Gummadi, K.~P. (2019).
\newblock Fairness constraints: A flexible approach for fair classification.
\newblock {\em The Journal of Machine Learning Research}, 20(1):2737--2778.

\bibitem[Zemel et~al., 2013]{zemel2013learning}
Zemel, R., Wu, Y., Swersky, K., Pitassi, T., and Dwork, C. (2013).
\newblock Learning fair representations.
\newblock In {\em International conference on machine learning}, pages
  325--333. PMLR.

\bibitem[Zeng et~al., 2022a]{zeng2022bayes}
Zeng, X., Dobriban, E., and Cheng, G. (2022a).
\newblock Bayes-optimal classifiers under group fairness.
\newblock {\em arXiv preprint arXiv:2202.09724}.

\bibitem[Zeng et~al., 2022b]{zeng2022fair}
Zeng, X., Dobriban, E., and Cheng, G. (2022b).
\newblock Fair {Bayes}-optimal classifiers under predictive parity.
\newblock In {\em Advances in Neural Information Processing Systems}.

\bibitem[Zhang et~al., 2018]{zhang2018mitigating}
Zhang, B.~H., Lemoine, B., and Mitchell, M. (2018).
\newblock Mitigating unwanted biases with adversarial learning.
\newblock In {\em Proceedings of the 2018 AAAI/ACM Conference on AI, Ethics,
  and Society}, pages 335--340.

\bibitem[Zhang and Long, 2021]{zhang2021assessing}
Zhang, Y. and Long, Q. (2021).
\newblock Assessing fairness in the presence of missing data.
\newblock {\em Advances in neural information processing systems},
  34:16007--16019.

\end{thebibliography}
